\documentclass[]{interact}

\usepackage[numbers,sort&compress]{natbib}% Citation support using natbib.sty
\bibpunct[, ]{[}{]}{,}{n}{,}{,}% Citation support using natbib.sty
\renewcommand\bibfont{\fontsize{10}{12}\selectfont}% Bibliography support using natbib.sty

\theoremstyle{plain}% Theorem-like structures provided by amsthm.sty

\usepackage[utf8]{inputenc} % allow utf-8 input
\usepackage[T1]{fontenc}    % use 8-bit T1 fonts
\usepackage{hyperref}       % hyperlinks
\usepackage{url}            % simple URL typesetting
\usepackage{booktabs}       % professional-quality tables
\usepackage{amsfonts}       % blackboard math symbols
\usepackage{nicefrac}       % compact symbols for 1/2, etc.
\usepackage{microtype}      % microtypography
\usepackage{xcolor}         % colors

\usepackage{amsmath,amsfonts}
\usepackage{algorithmic}
\usepackage{algorithm}
\usepackage{array}
\usepackage{subcaption}
\usepackage{textcomp}
\usepackage{stfloats}
\usepackage{url}
\usepackage{verbatim}
\usepackage{graphicx}
\usepackage{tablefootnote}
\usepackage{multirow}

\usepackage{amsmath,amssymb,amsfonts,amsthm}
\usepackage{mathtools}

\usepackage{algorithm}
\usepackage{algorithmic}

\def\xx{\mathbf{x}}
\def\yy{\mathbf{y}}
\def\cD{\mathcal{D}}

\DeclareMathOperator*{\argmin}{arg\,min}

\providecommand{\norm}[1]{\left\lVert#1\right\rVert}
\providecommand{\lin}[1]{\ensuremath{\left\langle #1 \right\rangle}}
\providecommand{\EE}[2]{{\mathbb E}_{#1}\left.#2\right. }

\def\R{\mathbb{R}}
\def\E{\mathbb{E}}

\def\I{{\mathcal I}}
\def\F{{\mathcal F}}
\def\O{\mathcal O}

\def\bbx{\bar{\mathbf x}}

\def\bbg{\bar{\mathbf g}}

\def\1{{\bf 1}}

\def\bx{{\mathbf x}}

\def\bz{{\mathbf z}}
\def\bg{{\mathbf g}}

\def\by{{\mathbf y}}

\newtheorem{assumption}{Assumption}
\newtheorem{definition}{Definition}
\newtheorem{lemma}{Lemma}
\newtheorem{theorem}{Theorem}
\newtheorem{example}{Example}
\newtheorem{corollary}{Corollary}
\newtheorem{remark}{Remark}

\title{The Role of Local Steps in Local SGD}

\author{
	\name{Tiancheng Qin\textsuperscript{a}\thanks{CONTACT Tiancheng Qin. Email: tq6@illinois.edu}, S. Rasoul Etesami\textsuperscript{a}, C\'esar A. Uribe\textsuperscript{b}}
	\affil{\textsuperscript{a}Department of Industrial and Enterprise Systems Engineering, Coordinated Science Laboratory, University of Illinois at Urbana-Champaign, Urbana, US; \textsuperscript{b}Department of Electrical and Computer Engineering, Rice University, Houston, US}
}

\begin{document}

	\maketitle

	\begin{abstract}
		We consider the distributed stochastic optimization problem where $n$ agents want to minimize a global function given by the sum of agents' local functions and focus on the heterogeneous setting when agents' local functions are defined over non-i.i.d. datasets. We study the Local SGD method, where agents perform a number of local stochastic gradient steps and occasionally communicate with a central node to improve their local optimization tasks. We analyze the effect of local steps on the convergence rate and the communication complexity of Local SGD. In particular, instead of assuming a fixed number of local steps across all communication rounds, we allow the number of local steps during the $i$-th communication round, $H_i$, to be different and arbitrary numbers. Our main contribution is to characterize the convergence rate of Local SGD as a function of $\{H_i\}_{i=1}^R$ under various settings of strongly convex, convex, and nonconvex local functions, where $R$ is the total number of communication rounds. Based on this characterization, we provide sufficient conditions on the sequence $\{H_i\}_{i=1}^R$ such that Local SGD can achieve linear speedup with respect to the number of workers.
		Furthermore, we propose a new communication strategy with increasing local steps superior to existing communication strategies for strongly convex local functions. On the other hand, for convex and nonconvex local functions, we argue that fixed local steps are the best communication strategy for Local SGD and recover state-of-the-art convergence rate results. Finally, we justify our theoretical results through extensive numerical experiments.
	\end{abstract}
	
	
	\section{Introduction}

	Stochastic Gradient Descent (SGD) is one of the most commonly used algorithms for parameter optimization of machine learning models. SGD tries to minimize a function $f$ by iteratively updating parameters as: $\mathbf{x}^{t+1}  = \mathbf{x}^{t} - \eta_t \hat{\mathbf{g}}^{t}$, where $\hat{\mathbf{g}}^{t}$ is a stochastic gradient of $f$ at $ \mathbf{x}^{t}$ and $\eta_t$ is the learning rate. However, given the massive scale of many modern ML models and datasets, and taking into account data ownership, privacy, fault tolerance, and scalability, distributed training approaches have recently emerged as a suitable alternative over centralized ones, e.g., parameter server~\cite{dean2012large}, federated learning~\cite{konevcny2016federated,mcmahan2016federated,kairouz2019advances,rieke2020future}, decentralized stochastic gradient descent ~\cite{lian2017can,tang2018d,koloskova2019decentralized,assran2018asynchronous}, decentralized momentum SGD~\cite{yu2019linear}, decentralized ADAM~\cite{nazari2019dadam}, among others~\cite{lu2020moniqua,tang2019doublesqueeze,bertsekas2015parallel}.

	A naive distributed generalization of SGD consists of having multiple agents computing stochastic gradients distributedly, with a central node or fusion center, where local gradients are aggregated and sent back to the agents at every iteration. However, communicating at each iteration induces a large communication overhead, where at each iteration of the algorithm, all agents need to send their gradients to the central node. Then the central node needs to send the agents the aggregated information. Local SGD (also known as Federated Averaging) presents a suitable solution to the problem~\cite{mcmahan2017communication,stich2018local,zinkevich2010parallelized,wang2018cooperative,koloskova2020unified,lu2020mixml}. Specifically, in Local SGD, each agent independently runs SGD locally for a number of steps and then aggregates by a central node from time to time only. The main advantage of Local SGD is that multiple local updates would likely move the model parameters much faster to the optimal solution in each communication round, thus effectively reducing the communication overhead at the cost of more local computations.
	
	On the other hand, it remains a delicate problem to choose the number of local steps during each communication round in Local SGD, as too few local steps would result in poor communication efficiency, while too many local steps would lead to slow convergence or even non-convergence of the algorithm. The problem is further complicated by the various scenarios the algorithm is facing, including different types of local objective functions, $i.e.$, strongly convex, general convex or nonconvex functions, as well as whether all agents have the same objective function (the homogeneous case)~\cite{khaled2019tighter,stich2019error,spiridonoff2020local} or different objective functions (the heterogeneous case)~\cite{gorbunov2021local,karimireddy2020scaffold,khaled2020tighter,qu2020federated,woodworth2020minibatch}. In this paper we focus on the more general heterogeneous case and study strongly convex, general convex and nonconvex local functions respectively.
	
	\subsection{Related Work} 
	
	For the case of homogeneous local functions, i.e., when all agents have the same objective function, it was shown in~\cite{khaled2019tighter,stich2019error} that using $\O(n\text{ polylog}(T))$ communication rounds, one can achieve convergence rate $\O(\frac{1}{nT})$ for Local SGD with strongly convex functions, where $n$ is the number of agents and $T$ is the number of iterations (or local gradient steps).

	A number of recent works have focused on the convergence analysis of Local SGD in heterogeneous setting~\cite{gorbunov2021local,karimireddy2020scaffold,khaled2020tighter,qu2020federated,woodworth2020minibatch}. It is shown that $\O(\frac{1}{nT})$ is both a lower and upper bound for the convergence rate of Local SGD for strongly convex objective functions~\cite{karimireddy2020scaffold,qu2020federated}. Moreover, it is known that $\O(\frac{1}{\sqrt{nT}})$ is both a lower and upper bound for the convergence rate of Local SGD for general convex and nonconvex objective functions~\cite{woodworth2020minibatch,khaled2020tighter}. These two convergence rates are often referred to as \emph{linear speedup with respect to the number of agents $N$} for strongly convex and convex/nonconvex objective functions, respectively. The name \emph{linear speedup} comes from the implication that with $N$ agents, the algorithm converges $N$ times faster than with just $1$ agent~\cite{qu2020federated}. Furthermore, for general convex and nonconvex local functions it is shown that Local SGD can achieve linear speedup with $\O(n^{\frac{3}{4}}T^{\frac{3}{4}})$ communication rounds~\cite{karimireddy2020scaffold,khaled2020tighter}. For strongly convex local functions, the results in~\cite{koloskova2020unified} implies that Local SGD can achieve linear speedup with $\O(\sqrt{nT}\text{ polylog}(nT))$ communication rounds without the bounded gradient assumption;~\cite{qu2020federated} showed that linear speedup can be achieved with $\O(\sqrt{nT})$ communication rounds, however, their analysis requires the bounded gradient assumption, which is unrealistic in certain cases (see, e.g.~\cite{khaled2020tighter}).

	On the other hand, while most of the works mentioned above assume a fixed number of local steps across all communication rounds, several recent works have proposed different communication strategies for Local SGD to reduce communication costs further. Specifically, in the homogeneous setting,~\cite{wang2019adaptive} proposed an adaptive communication strategy that gradually increases communication frequency for training neural networks.~\cite{haddadpour2019local} analyzed loss functions that satisfy the Polyak-Łojasiewicz condition and proposed decreasing communication frequency. Recently,~\cite{spiridonoff2021communication} proposed a linearly increasing number of local steps for strongly convex objective functions and theoretically showed its better communication efficiency. This result has been further generalized in~\cite{qin2021communication} to the network settings. In the heterogeneous setting,~\cite{lin2018don} proposed decreasing communication frequency such that a number of fully synchronized SGD steps are performed, followed by Local SGD with a fixed number of local steps. On the contrary,~\cite{li2019communication} proposed increasing communication frequency such that the number of local steps decreases exponentially until it reaches unit local steps.
	
	\subsection{Contributions and Organization}
	
	In this paper, we study the role of local steps in Local SGD in a heterogeneous setting. In particular, we allow the number of local steps during the $i$-th communication round, $H_i$, to be different integer numbers, and characterize the convergence rate of Local SGD with respect to the sequence $\{H_i\}_{i=1}^R$, where $R$ is the total number of communication rounds. Such a characterization enables us to study the convergence rate of Local SGD for any general communication pattern. We summarize our contributions as follows:
	
	\begin{itemize}
		\item We characterize the convergence rate of Local SGD explicitly as a function of $\{H_i\}_{i=1}^R$ under various settings of strongly convex, convex, and nonconvex local functions.
		\item We provide sufficient conditions on the sequence $\{H_i\}_{i=1}^R$ such that Local SGD can achieve linear speedup with respect to the number of agents, i.e., $\O(\frac{1}{nT})$ convergence rate for strongly convex local functions and $\O(\frac{1}{\sqrt{nT}})$ convergence rate for general convex or nonconvex local functions, that covers broad classes of communication strategies.
		\item For strongly convex local functions, we propose a new communication strategy for the Local SGD with an increasing number of local steps and show it can achieve linear speedup convergence rate with $\O(\sqrt{nT})$ communication rounds without any assumption on the boundedness of the gradients. To our knowledge, this is the first result of the linear speedup of Local SGD with $\O(\sqrt{nT})$ communication rounds that do not require the bounded gradient assumption. We also validate the superiority of the communication strategy through numerical experiments.
		\item Based on our convergence rate characterization, we argue that using fixed local steps is the best communication strategy for Local SGD in the case of convex and nonconvex local functions. Our results imply that  Local SGD can achieve a linear speedup convergence rate with $\O(n^{\frac{3}{4}}T^{\frac{3}{4}})$ communication rounds, which matches the best-known results in this setting~\cite{karimireddy2020scaffold,khaled2020tighter}. Moreover, we show through numerical experiments that this bound on the number of communication rounds to achieve linear speedup is almost tight.
	\end{itemize}

	The paper is organized as follows. Section~\ref{sec:problem} describes the problem statement. Section~\ref{sec:convex-results} states our main results for the case of strongly convex and convex objective functions. Section \ref{sec:nonconvex} extends our convergence rate analysis to the case of nonconvex functions. Simulation results are given in Section \ref{sec:numerical}, followed by conclusions and future directions in Section~\ref{sec:conclusion}. For ease of presentation, all the proof details are deferred to the supplementary materials.

	\section{Problem Formulation}\label{sec:problem}

	We consider the distributed stochastic optimization problem with a set of $[n]=\{1, \ldots, n\}$ agents, where each agent $i\in [n]$ holds a local objective function $f_i \colon \R^d \to \R$ that can be expressed in a stochastic form 
	\begin{align}\label{eq:f_i}
		f_i(\xx)= \E_{\xi_i \sim \cD_i} F_i(\xx,\xi_i).
	\end{align}
	Here, $\xx\in \mathbb{R}^d$ is the optimization variable, and $\cD_i$ denotes the distribution of random variable $\xi_i$ over the parameter sample space~$\Omega_i$ for agent $i$. The agents' goal is to minimize the global objective function $f \colon \R^d \to \R$ given by the average sum of all the local functions or, equivalently, solve the following unconstrained optimization problem 
	\begin{align}\label{eq:f}
		f^\star := \min_{\xx \in \R^d} \big\{  f(\xx)=\frac{1}{n} \sum_{i=1}^n f_i(\xx)\big\}, 
	\end{align}
	by performing local gradient steps and occasionally communicating with a central node to leverage the samples obtained by the other agents.

	We assume throughout the paper that $f(\xx)$ is bounded below by $f^\star$ (i.e., a global minimum exists), $f_i(\xx,\xi_i)$ is $L$-smooth for every $i\in [n]$, and $\nabla F_i(\xx,\xi_i)$ is an unbiased stochastic gradient of $f_i(\xx)$, which by now are standard assumptions in the context of federated learning~\cite{khaled2020tighter,karimireddy2020scaffold}.  Moreover, for some of our results, we will require functions $f_i$ to be $\mu$-strongly convex with respect to the parameter $\xx$ as defined next.
	\begin{assumption}\label{a:strong}
		We say $f_i \colon \R^d \to \R$ is $\mu$-(strongly) convex for some $\mu \geq 0$ if for all $\xx,\yy \in \R^d$, we have
		\begin{align}\nonumber
			f_i(\xx)-f_i(\yy) + \frac{\mu}{2}\norm{\xx-\yy}^2_2 \leq \lin{\nabla f_i(\xx),\xx-\yy}. 
		\end{align}
		If $\mu=0$, then $f_i$ is convex but not strongly-convex.
	\end{assumption}
	
	Next, as in~\cite{khaled2020tighter}, we consider the following definition, which allows us to measure the heterogeneity among local functions.
	\begin{definition}\label{eq:var_opt}
		Assume~\eqref{eq:f} admits a unique optimal solution $\xx^\star = \argmin f(\xx)$. We define
		\begin{align}\nonumber
			\bar{\sigma}^2= \frac{1}{n}\sum_{i = 1}^n\EE{\xi_i}{\norm{\nabla F_i(\xx^*, \xi_i) }^2_2}.
		\end{align}
	\end{definition}
	It follows that for all non-degenerate sampling distribution $\cD_i$, $\bar{\sigma}^2$ is well-defined and finite and serves as a natural measure of variance in local methods. However, for nonconvex objective functions where a unique $\bx^*$ may not exist, as in~\cite{karimireddy2020scaffold}, we consider the following assumption of bounded gradient dissimilarity.
	\begin{assumption}\label{a:nonconvex}(bounded gradient dissimilarity)
		We say that the local functions $f_i$ satisfy $(G,B)$-bounded gradient dissimilarity (or for short $(G,B)$-\boldmath{$BGD$} \unboldmath) if there exist constants $G\ge0$ and $B\ge1$ such that
		\begin{align}\nonumber
			\frac{1}{n}\sum_{i = 1}^n\|\nabla f_i(\bx)\|^2\le G^2+B^2\|\nabla f(\bx)\|^2,\ \forall \bx.
		\end{align}
		We also assume $\nabla F_i(\xx,\xi_i)$ is an unbiased stochastic gradient of $f_i(\xx)$ with variance bounded by $\sigma^2$.
	\end{assumption}

	\subsection{Local Stochastic Gradient Descent}
	
	A popular method for solving~\eqref{eq:f} in a distributed manner is the local stochastic gradient descent (Local SGD) method. In Local SGD, each agent performs local gradient steps, and a central node will compute the average of all agents' iterates every once in a while to guide agents' iterates toward consensus. Let us denote the total number of iterations in Local SGD by $T$ and the set of communication instances by $\I \subseteq [T]$. Then, in every iteration $t\in [T]$ of the Local SGD i) each agent $i\in [n]$ performs stochastic gradient descent update on its local objective function, and ii) if $t$ is a communication time, i.e., $t\in \I$, each agent $i\in [n]$ sends its current local solution $\xx_i^{(t)}$ to the central node and receives the average of all agents' local solutions. The pseudo-code for the Local SGD algorithm is summarized in Algorithm \ref{alg1}. 
	
	\begin{algorithm}[t]
		\caption{Local SGD}
		\begin{algorithmic}[1]\label{alg1}
			\STATE Input $\bx_i^{(0)} = \bx^{(0)}$ for $i \in [n]$, total number of iterations $T$, the step-size sequence $\{\eta_t\}_{t=0}^{T-1}$, the set of communication time instances $\I = \{\tau_{i}\}_{i=0}^R$.
			\FOR{$t=0,\ldots,T-1$}
			\FOR{$i=1,\ldots,n$}
			\STATE Sample $\xi_i^{(t)}$, compute $\bg_i^t=\nabla F_i(\xx_i^{(t)}\!, \xi_i^{(t)})$
			\IF{$t+1\in \I$}
			\STATE $\xx_i^{(t+1)} = \frac{1}{n}\sum_{j = 1}^n(\xx_j^{(t)}-\eta_t\bg_i^t)$
			\ELSE
			\STATE $\xx_i^{(t+1)} = \xx_i^{(t)} - \eta_t \bg_i^t$
			\ENDIF
			\ENDFOR
			\ENDFOR
		\end{algorithmic}
	\end{algorithm}
	
	Finally, we consider the following definition of communication intervals in the Local SGD.
	
	\begin{definition}
		Given communication time instances $\I =\{\tau_i\} _{i=1}^R$, we let $H_i=\tau_{i}-\tau_{i-1}$ be the length of the $i$-th communication interval, i.e., the number of local steps between the $(i-1)$-th and $i$-th communications. Moreover, for any time instance $t\in [\tau_i, \tau_{i+1})$, we define $k(t)=i$. In other words, $k(t)$ is the index such that $\tau_{k(t)}\le t<\tau_{k(t)+1}$.
	\end{definition}  
	
	Our main objective in this work is to characterize the convergence rate of Algorithm~\ref{alg1} with respect to the sequence of the local steps $\{H_i\}_{i=1}^R$ as defined above, when applied to the optimization problem~\eqref{eq:f}.
	
	%In the subsequent sections, we provide an explicit relationship between the convergence rate of Algorithm \ref{alg1} and the length of communication intervals.

	%%%%%%%%%%%%%%%%%%%%%%%%%%%%%%%%%%%%%%%%%%%%%%%%%%%%%%%%%%%%%%%%%%%%%%%
	%%%%%%%%%%%%%%%%%%%%%%%%%%%%%%%%%%%%%%%%%%%%%%%%%%%%%%%%%%%%%%%%%%%%%%%
	%%%%%%%%%%%%%%%%%%%%%%%%%%%%%%%%%%%%%%%%%%%%%%%%%%%%%%%%%%%%%%%%%%%%%%%
	\section{Convergence Results for Local SGD}\label{sec:convex-results}
	%%%%%%%%%%%%%%%%%%%%%%%%%%%%%%%%%%%%%%%%%%%%%%%%%%%%%%%%%%%%%%%%%%%%%%%
	%%%%%%%%%%%%%%%%%%%%%%%%%%%%%%%%%%%%%%%%%%%%%%%%%%%%%%%%%%%%%%%%%%%%%%%
	%%%%%%%%%%%%%%%%%%%%%%%%%%%%%%%%%%%%%%%%%%%%%%%%%%%%%%%%%%%%%%%%%%%%%%%
	
	In this section, we state our main result for the case of strongly convex and convex functions. To that end, let $\bbx^{(t)}$ and $\bbg^{(t)}$ be the average of agents' iterates and the average of their stochastic gradients at time $t$, respectively, i.e.,  
	\begin{align*}
		\bbx^{(t)}=\frac{1}{n}\sum_{i = 1}^n \bx_i^{(t)},\ \ \ \bbg^{(t)}=\frac{1}{n}\sum_{i = 1}^n \nabla F_i(\bx_i^{(t)},\xi_i^{(t)}).
	\end{align*}
	Moreover, define the following parameters
	\begin{align}\nonumber
		r_t = \E\|\bbx^{(t)} -\bx^*\|^2,\quad
		V_t = \frac{1}{n}\E \sum_{i=1}^n\|\xx_i^{(t)}-\bbx^{(t)}\|^2,\quad
		e_t = \E [f(\bbx^{(t)})]-f(\bx^*),
	\end{align}
	which represent the expected distance of the averaged iterates at time $t$ to the optimum solution, the expected consensus error among agents at time $t$, and the expected optimality gap at time $t$.
	
	\subsection{Convergence Result for Strongly Convex Functions}

	\begin{theorem}\label{theo:stronglyconvex}
		Let Assumption \ref{a:strong} hold with $\mu>0$. Then, the sequence generated by Algorithm~\ref{alg1} with stepsize $\eta_t=\frac{2}{\mu(\beta+t)}$, and any sequence of communication intervals $\{H_i\}_{i=1}^R$ and parameter $\beta$ such that
		\begin{align}\nonumber
			H_i\le \frac{\mu (\beta + \sum_{j = 1}^{i-1}H_j)}{12L}\ \forall i,
		\end{align}	
		has the following property:
		\begin{align}\label{eq:sronglyconvex}
			r_T\!\leq\! \frac{(\beta\!-\!1)^2}{T^2}r_0\!+\!\frac{12\bar{\sigma}^2}{n\mu^2T}\!+\!\frac{144L\bar{\sigma}^2}{\mu^3T^2}\sum_{i=1}^{R}\frac{H^3_{i}}{\sum_{j=1}^{i-1}H_j+\beta}.
		\end{align}
		where $L$ is the smoothness constant, $R$ is the number of communication rounds, and $\beta$ is a constant that can be tuned by the Local SGD algorithm to balance the first and third term in~\eqref{eq:sronglyconvex}. 
	\end{theorem}

	An immediate corollary of Theorem \ref{theo:stronglyconvex} is the set of sufficient conditions on the sequence $\{H_i\}_{i=1}^R$ that leads to linear speedup in the convergence of Algorithm \ref{alg1}. 
	
	\begin{corollary}\label{corr:strongly-convex}
		Assume that $T\ge N$. Let the sequence of local steps $\{H_i\}_{i=1}^R$ have the following properties:
		\begin{align*}
			(1) H_i\le \frac{\mu (\beta + \sum_{j = 1}^{i-1}H_j)}{12L},\ \forall i\quad
			(2)\sum_{i=1}^R H_i=T\quad
			(3)\sum_{i = 1}^R \frac{H_i^3}{\sum_{j = 1}^{i-1}H_i+\beta}=\mathcal{O}(\frac{T}{n}).
		\end{align*}
		Then, the sequence generated by in Algorithm~\ref{alg1} has the following property $r_T=\mathcal{O}(\frac{1}{nT})$. 
	\end{corollary}
	
	Next, we analyze two special communication strategies, one with a fixed number of local steps and the other with an increasing number of local steps.
	\begin{example}\label{exp:1}
		Consider the case when number of local steps is fixed, i.e., $H_i = \frac{T}{R}, \forall i$. In order to achieve linear speedup, condition 3) in Corollary \ref{corr:strongly-convex} is equivalent to $\frac{T^2}{R^2}\sum_{i=1}^R\frac{1}{i} =\mathcal{O}(\frac{T}{n})$, which can be translated into the number of communication rounds $R=\O(\sqrt{nT}\text{ polylog}(nT))$. This matches the result given in~\cite{koloskova2020unified}.
	\end{example}
	
	\begin{example}\label{exp:2}
		Consider the communication strategy with increasing number of local steps $H_i = \lfloor ai^s \rfloor, \forall i$, for some parameter $a>0$ and $s>0$. To achieve linear speedup, one can choose $a = \mathcal{O}(n^{-\frac{s+1}{2}}T^{\frac{1-s}{2}})$, in which case the number of communication rounds becomes $R =\mathcal{O}((\frac{T}{a})^{\frac{1}{s+1}}) = \mathcal{O}(\sqrt{nT})$. This would satisfy both conditions 2) and 3) in Corollary \ref{corr:strongly-convex}, and we can choose $\beta =  a\lceil \frac{24L}{\mu}\rceil^s\cdot \frac{12L}{\mu}+1$ in order to satisfy condition 1)\footnote{Proof in Appendix A.5}. Therefore, using Corollary \ref{corr:strongly-convex}, following this communication strategy, Local SGD can achieve linear speedup convergence rate with $\O(\sqrt{nT})$ communication rounds.
	\end{example}
	
	\begin{remark}
		\label{rem:1}
		The communication strategy with an increasing number of local steps as in Example \ref{exp:2} exhibits better communication efficiency than a fixed number of local steps. To the best of our knowledge, this is the first result for the linear speedup of Local SGD with $\O(\sqrt{nT})$ communication rounds that do not require any assumption on the boundedness of the gradients. 
	\end{remark}
	\subsection{Convergence Result for Convex Functions}
	
	In this part, we relax the assumption of strong convexity on the local function to merely convex functions and analyze the convergence rate of Algorithm \ref{alg1} in terms of the number of local steps. 
	
	\begin{theorem}
		\label{theo:convex}
		Let Assumption \ref{a:strong} be satisfied with $\mu=0$ and set a stepsize as $\eta_t =c\sqrt{\frac{n}{T}}$, $\forall c>0$. Moreover, set the communication intervals to satisfy $H_i\le \frac{1}{7L\eta}=\frac{\sqrt{T}}{7Lc\sqrt{n}},\forall i$. Thus, the iterates generated by Algorithm~\ref{alg1} have the following property:
		\begin{align}
			\label{eq:convexrate}
			\frac{1}{T}\sum_{t=0}^{T-1}e_t\le \frac{2r_0+6c^2\bar{\sigma}^2}{c\sqrt{nT}}+\frac{24L\bar{\sigma}^2c^2n}{T^2}\sum_{i = 1}^RH_i^3.
		\end{align}
	\end{theorem}

	An immediate corollary of Theorem \ref{theo:convex} is a sufficient condition on the sequence $\{H_i\}_{i=1}^R$ that leads to linear speedup in the convergence of Algorithm \ref{alg1}. 
	
	\begin{corollary}
		\label{corr:convex}
		Assume that $T\ge N^3$. In order to achieve \emph{Linear Speedup},  $\frac{1}{T}\sum_{t=0}^{T-1}e_t=\mathcal{O}(\frac{1}{\sqrt{nT}})$, it is enough to select the local steps sequence $\{H_i\}_{i=1}^R$ such that
		\begin{align*}
			(1)H_i\le \frac{1}{7L\eta}=\frac{\sqrt{T}}{7Lc\sqrt{n}},\forall i\quad,
			(2)\sum_{i = 1}^R H_i = T\quad,
			(3)\sum_{i = 1}^R H_i^3=\mathcal{O}(\frac{T^{\frac{3}{2}}}{n^{\frac{3}{2}}}).
		\end{align*}
	\end{corollary}
	
	\begin{remark}
		A closer look at the bound~\eqref{eq:convexrate} reveals that in order to minimize the error bound of Local SGD, the sequence $\{H_i\}_{i=1}^R$ should minimize $\sum_{i = 1}^R H_i^3$ subject to $\sum_{i = 1}^R H_i = T$. This leads to the communication strategy of a fixed number of local steps, i.e., $H_i = \frac{T}{R}$. Therefore, for convex local functions, the fixed number of local steps seems to be the best communication strategy for Local SGD. Moreover, from Corollary \ref{corr:convex}, we immediately get that in order to achieve \emph{linear speedup}, the number of communication rounds should be $R=\mathcal{O}((nT)^{3/4})$, which matches the best-known results in this setting ~\cite{khaled2020tighter}.
	\end{remark}
	
	\subsection{Convergence Result for Non-Convex Functions}\label{sec:nonconvex}
	
	In this section, we focus on the class of nonconvex local functions. However, we need to impose the additional $(G,B)$-BGD assumption to analyze the convergence rate versus communication complexity trade-off. To state our main result, let us define 
	\begin{align}\nonumber
		h_t = \|\nabla f(\bbx ^{(t)})\|^2,
	\end{align}
	which is the gradient norm of the average iterates in the Local SGD. Then, we have the following theorem.
	
	\begin{theorem}\label{theo:nonconvex}
		Let Assumption \ref{a:nonconvex} hold, fix a stepsize $\eta_t =c\sqrt{\frac{n}{T}}$, $\forall c>0$, and set a sequence of communication intervals that satisfy $H_i\le \frac{1}{7LB\eta}=\frac{\sqrt{T}}{7LBc\sqrt{n}},\forall i$. Then, the sequence generated by Algorithm~\ref{alg1} has the following property:
		\begin{align}
			\label{eq:nonconvexrate}
			\frac{1}{T}\!\sum_{t=0}^{T-1}h_t\!\le\! \frac{8e_0+4c^2\sigma^2}{c\sqrt{nT}}\!+\!\frac{48L^2(\sigma^2\!+\!G^2)c^2n}{T^2}\sum_{i = 1}^RH_i^3.
		\end{align}
	\end{theorem}

	As a corollary of Theorem \ref{theo:convex}, we obtain the following set of sufficient conditions on the sequence $\{H_i\}_{i=1}^R$ that leads to linear speedup in the convergence of Algorithm \ref{alg1}. 
	
	\begin{corollary}
		\label{corr:nonconvex}
		Assume that $T\ge N^3$. In order to achieve \emph{Linear Speedup, } $\frac{1}{T}\sum_{t=0}^{T-1}h_t=\mathcal{O}(\frac{1}{\sqrt{nT}})$, it is enough to select the local steps sequence $\{H_i\}_{i=1}^R$ such that
		\begin{align*}
			(1)H_i\le \frac{1}{7LB\eta}=\frac{\sqrt{T}}{7LBc\sqrt{n}},\forall i\quad,
			(2)\sum_{i = 1}^R H_i = T\quad,
			(3)\sum_{i = 1}^R H_i^3=\mathcal{O}(\frac{T^{\frac{3}{2}}}{n^{\frac{3}{2}}}).
		\end{align*}
	\end{corollary}

	\begin{remark}
		\label{rem:3}
		By taking a closer look at the bound~\eqref{eq:nonconvexrate}, it is easy to see that in order to minimize the error bound of Local SGD, the local steps sequence $\{H_i\}_{i=1}^R$ should minimize $\sum_{i = 1}^R H_i^3$ subject to $\sum_{i = 1}^R H_i = T$. This leads to the communication strategy of a fixed number of local steps, i.e., $H_i = \frac{T}{R}$. Therefore, for nonconvex local functions, we conclude that a fixed number of local steps is the best communication strategy for Local SGD. Moreover, from Corollary \ref{corr:nonconvex}, we immediately get that in order to achieve linear speedup, the number of communication rounds should be $R=\mathcal{O}((nT)^{3/4})$, which matches the best-known results in this setting~\cite{karimireddy2020scaffold}.
	\end{remark}
	
	\section{Numerical Results}\label{sec:numerical}
	
	This section shows the results for two sets of experiments on the MNIST dataset~\cite{lecun1998gradient} to validate our theoretical findings. We focus on strongly-convex loss functions for the first set of experiments, where we train a logistic regression model with $l_2$ regularization. We focus on nonconvex loss functions for the second set of experiments, where we train a small, fully connected neural network.
	
	\subsection{Logistic Regression Model for MNIST}
	In this set of experiments, we distribute the MNIST dataset to $n=20$ agents and apply Local SGD to train a multinomial logistic regression model with $l_2$ regularization. We first sort the data by digit label, then divide the dataset into $100$ shards and assign each of $20$ agents $5$ shards. Each agent will have examples of approximately five digits, reflecting moderately heterogeneous data sets. 
	
	We evaluate different communication strategies (i.e., various numbers of local steps when following communication strategy with a fixed number of local steps as in Example \ref{exp:1}, and $a=10,s=0.2$ when following communication strategy with an increasing number of local steps as in Example \ref{exp:2}) the corresponding communication rounds and iterations needed for the model to reach a 91.5\% accuracy on the MNIST test dataset. The simulation results are averaged over $5$ independent runs of the experiments and are shown in Figure~\ref{fig:1}, and Figure~\ref{fig:2}.
	
	For the set of hyperparameters, we use a training batch size of $8$, $l_2$ regularization parameter $\mu = 0.001$, $\beta=1000$ and set stepsize at iteration $t$ to be $\eta_t  = \frac{\beta}{t+\beta}\eta_0$, where the initial stepsize $\eta_0$ is chosen based on a grid search of resolution $10^{-3}$. 
	
	Figure~\ref{fig:1} shows the details of the runs of the experiment. Figure~\ref{fig:2} shows the summary of the runs. For example, the upper left yellow dot in Figure~\ref{fig:2} corresponds to the average of $5$ runs of Local SGD with constant $H_i = 1$, showing that with constant $H_i = 1$ it took the algorithm an average of $\sim 305$ communication rounds as well as total iterations to reach 91.5\% accuracy.
	\begin{figure}[]
		\centering
		\begin{subfigure}{0.48\linewidth}
			\includegraphics[width=\linewidth]{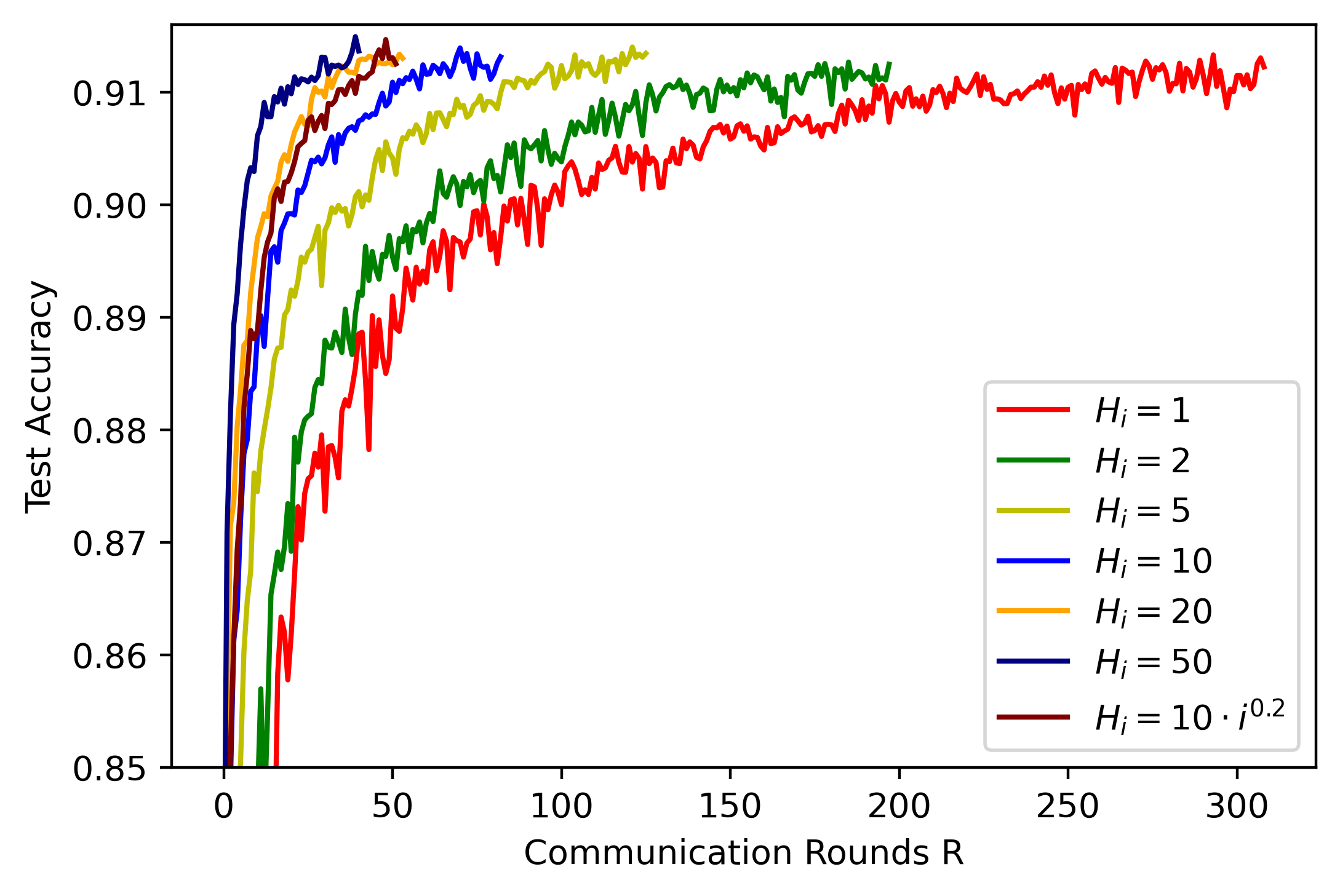}
			\caption{}
			\label{fig:1a}
		\end{subfigure}
		\begin{subfigure}{0.48\linewidth}
			\includegraphics[width=\linewidth]{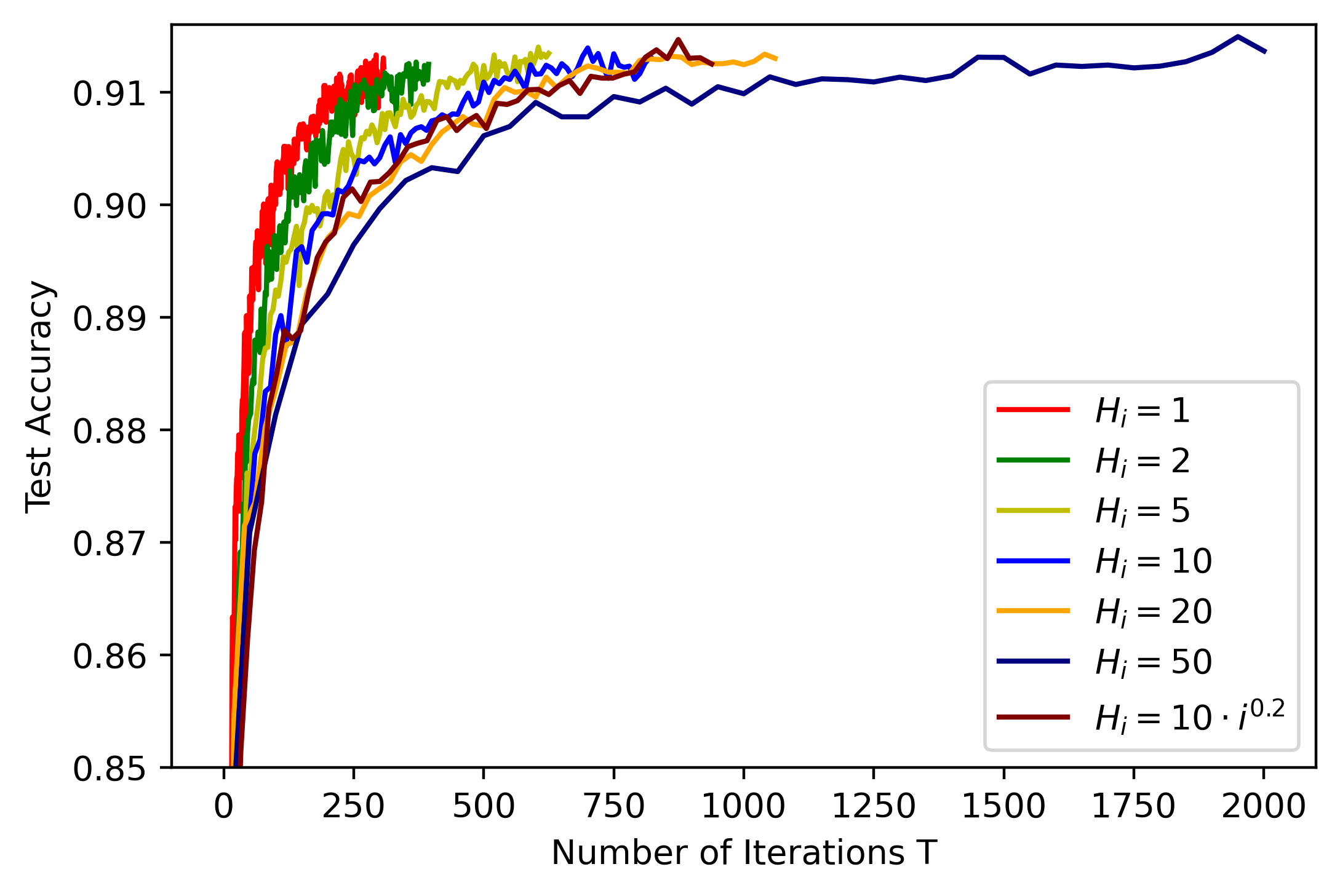}
			\caption{}
			\label{fig:1b}
		\end{subfigure}
		\caption{ Logistic regression for MNIST. \ref{fig:1a}: Test accuracy vs. communication rounds for different communication strategies. \ref{fig:1b}: test accuracy vs. the number of iterations for different communication strategies. Simulation results averaged over $5$ runs of the experiment.}
		\label{fig:1}
	\end{figure}
	\begin{figure}
		\centering
		\includegraphics[width=0.48\linewidth]{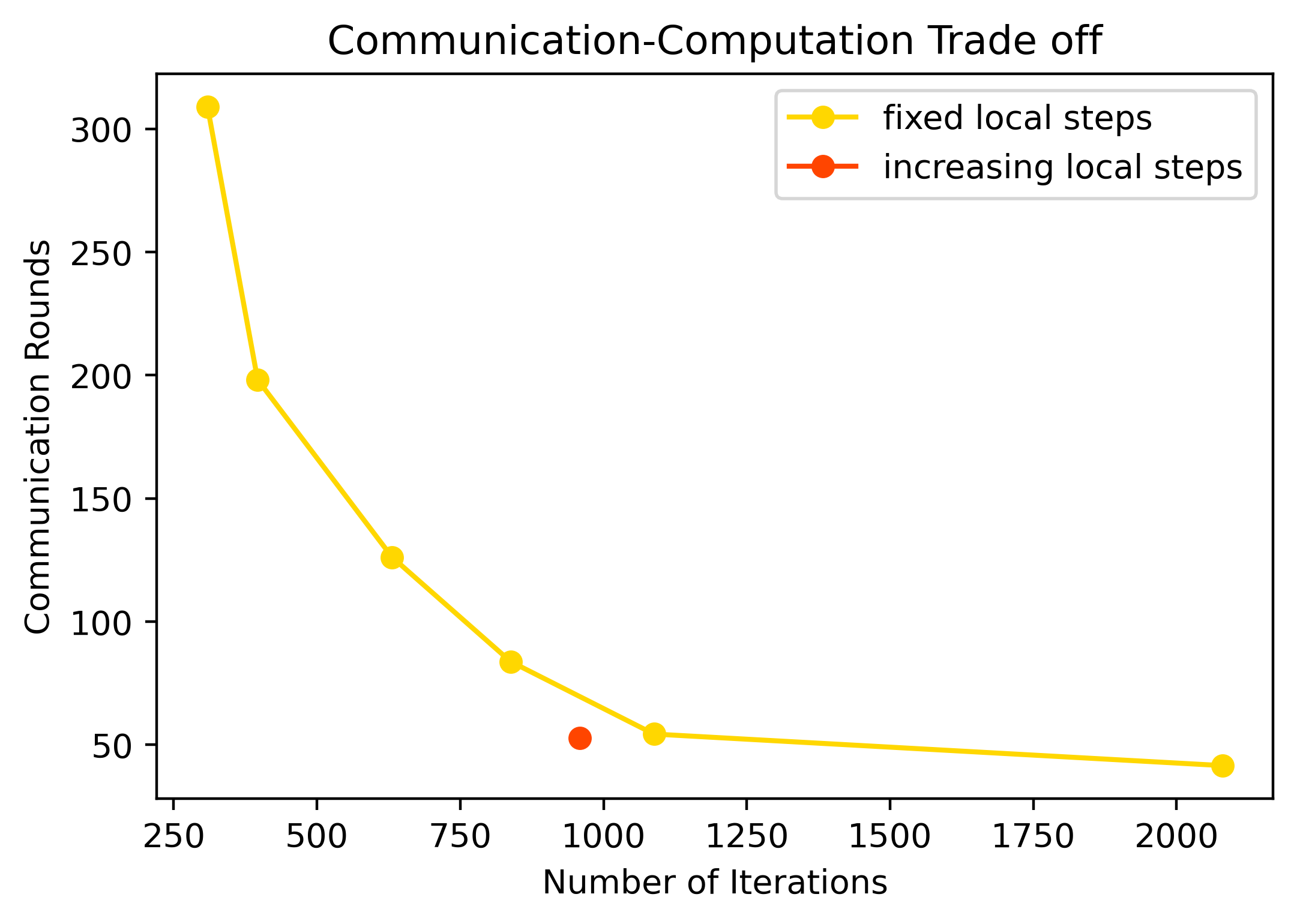}
		\caption{ Logistic regression for MNIST. Summary of the number of communication rounds and iterations needed for the model to reach a 91.5\% accuracy on the MNIST test dataset for different communication strategies. Yellow dots: fixed number of local steps as in Example \ref{exp:1}. Red dot: increasing number of local steps as in Example \ref{exp:2} with $a=10,s=0.2$. Simulation results averaged over $5$ runs of the experiment.}
		\label{fig:2}
	\end{figure}
	
	\smallskip
	\noindent
	{\bf Communication-Computation Trade-Off:} In general, we can observe a communication-computation trade-off such that with more local computation (corresponding to a larger number of iterations $T$), less communication is needed (corresponding to a smaller number of communication rounds $R$) for the model to reach a certain accuracy.
	
	\smallskip
	\noindent
	{\bf Better Communication Efficiency with Increasing Number of Local Steps:} As we can see from Figure \ref{fig:2}, the red dot lies to the bottom left of the yellow line, which shows that the communication strategy of an increasing number of local steps is indeed more communication efficient than a fixed number of local steps, thus validating our claim in Remark~\ref{rem:1}.

	\subsection{Neural Network for MNIST}
	In this set of experiments, we distribute the MNIST dataset to $n$ agents and apply Local SGD to train a fully-connected neural network (2NN) with 2-hidden layers with 50 units each using ReLu activations (42310 total parameters)\footnote{We have deliberately chosen to train a small neural network to avoid getting an overparameterized model, in which case the convergence rate of Local SGD would be different~\cite{qin2022faster}.}. We first sort the data by digit label, then divide the dataset into $n$ shards and assign each of $n$ agents $1$ shards. Each agent will have examples of approximately one digit, reflecting the most heterogeneous data sets. 
	
	We evaluate the speedup effect of the number of agents $n$ for different communication strategies. In particular, we set a fixed number of $T=20000$ iterations and run Local SGD for $T$ iterations with different communication strategies, a different number of agents $n$, and a different number of communication rounds $R$. After that, a speedup factor is derived by dividing the expected error of a single worker SGD at the final iterate $T$ by the expected error of Local SGD with different communication strategies and a different number of agents $n$ at the final iterate $T$. We plot the speedup curve in Figure~\ref{fig:3}. In the case of linear speedup, we should expect the dashed black line on the graph, corresponding to speedup $=\sqrt{n}$.
	
	We use a training batch size of $1$ and choose stepsize $\eta$ based on a grid search of resolution $10^{-3}$. The simulation results are averaged over $5$ independent runs of the experiments.

	\begin{figure}
		\centering
		\includegraphics[width=0.48\linewidth]{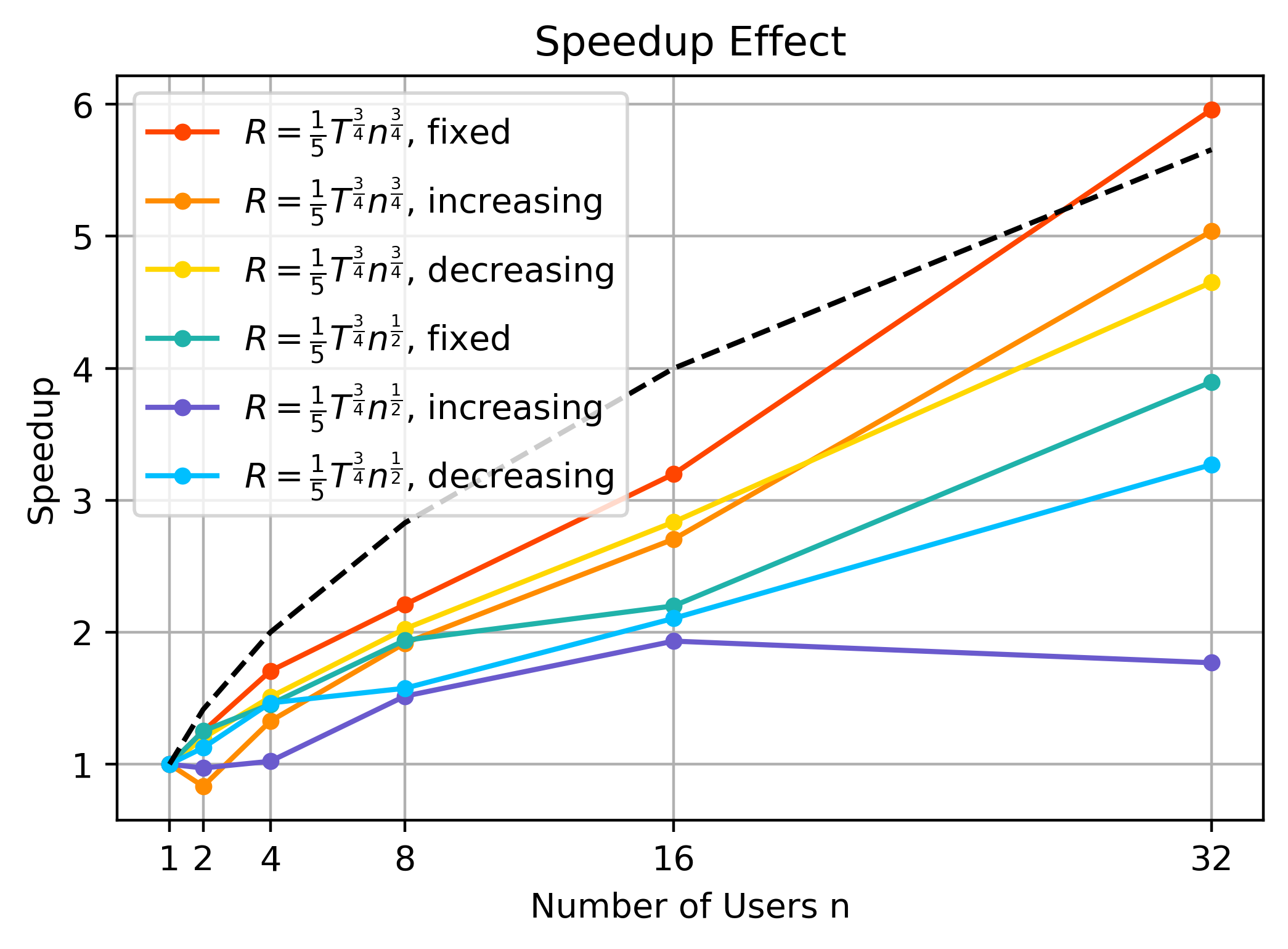}
		\caption{Speedup curve for Local SGD with different communication strategies. Fixed: fixed number of local steps subject to a total of $R=\frac{1}{5}T^{\frac{3}{4}}n^{\frac{3}{4}}$ or $R=\frac{1}{5}T^{\frac{3}{4}}n^{\frac{1}{2}}$ communication rounds. Increasing: $H_i\varpropto i^2$ subject to a total of $R=\frac{1}{5}T^{\frac{3}{4}}n^{\frac{3}{4}}$ or $R=\frac{1}{5}T^{\frac{3}{4}}n^{\frac{1}{2}}$ communication rounds. Decreasing: $H_i\varpropto (R-i)^2$ subject to a total of $R=\frac{1}{5}T^{\frac{3}{4}}n^{\frac{3}{4}}$ or $R=\frac{1}{5}T^{\frac{3}{4}}n^{\frac{1}{2}}$ communication rounds. The dashed black line corresponds to speedup $=\sqrt{n}$.}
		\label{fig:3}
	\end{figure}
	
	\smallskip
	\noindent
	{\bf Better Performance with Fixed Number of Local Steps:} We can observe from Figure~\ref{fig:3} that Local SGD with fixed number of local steps significantly outperforms its increasing or decreasing number of local steps counterparts in both settings of \mbox{$R=\frac{1}{5}T^{\frac{3}{4}}n^{\frac{3}{4}}$} (corresponding to sufficient communication) and \mbox{$R=\frac{1}{5}T^{\frac{3}{4}}n^{\frac{1}{2}}$} (corresponding to insufficient communication). This validates our claim in Remark~\ref{rem:3} that a fixed number of local steps is the best communication strategy for Local SGD for nonconvex local functions.
	
	\smallskip
	\noindent
	{\bf Almost Tight Bound for $R=\mathcal{O}((nT)^{3/4})$ to Achieve Linear Speedup:} Another observation from Figure~\ref{fig:3} is that while setting $R=\mathcal{O}((nT)^{3/4})$ and following a communication strategy of a fixed number of local steps, Local SGD successfully achieved linear speedup, as expected, decreasing $R$ by a factor of $n^{\frac{1}{4}}$ fails for Local SGD to achieve linear speedup, even with the best communication strategy of a fixed number of local steps. This suggests that the bound of $R=\mathcal{O}((nT)^{3/4})$ to achieve linear speedup is close to tight.
	
	\section{Conclusions}\label{sec:conclusion}
	
	In this paper, we analyzed the role of local steps in Local SGD in the heterogeneous data setting. We characterized the convergence rate of Local SGD as a function of the sequence of the local steps $\{H_i\}_{i=1}^R$ under various settings of strongly convex, convex, and nonconvex local functions. Based on this characterization, we gave sufficient conditions on the sequence $\{H_i\}_{i=1}^R$ that covers broad classes of communication strategies such that Local SGD can achieve linear speedup. Furthermore, for strongly convex local functions, we proposed a new communication strategy with increasing local steps that enjoy better performance than the vanilla fixed local steps communication strategy theoretically and in numerical experiments. We argued that fixed local steps are the best communication strategy for Local SGD and recover state-of-the-art convergence rate results for convex and nonconvex local functions. Such an argument is validated by numerical experiments, which showed that the results are almost tight. 
	
	As a future research direction, one can consider analyzing the role of local steps in other federated optimization methods, e.g., SCAFFOLD~\cite{karimireddy2020scaffold}, FedAC~\cite{yuan2020federated}. Moreover, generalizing our work to directed networks in which agents communicate with their neighbors rather than a central node is another interesting research problem, $e.g.$ for Stochastic Gradient Push algorithm~\cite{assran2019stochastic}. Also, we only considered the role of local steps in Local SGD with full agent participation; generalizing it to the partial participation setting is yet another interesting problem.

	\bibliographystyle{tfs}
	\bibliography{references}

	%%%%%%%%%%%%%%%%%%%%%%%%%%%%%%%%%%%%%%%%%%%%%%%%%%%%%%%%%%%%

	%%%%%%%%%%%%%%%%%%%%%%%%%%%%%%%%%%%%%%%%%%%%%%%%%%%%%%%%%%%%

	%%%%%%%%%%%%%%%%%%%%%%%%%%%%%%%%%%%%%%%%%%%%%%%%%%%%%%%%%%%%
	
	\newpage
	\appendix

	\section{Appendix I: Omitted Proofs}\label{sec:appx}
	
	\subsection{Proof of Theorem \ref{theo:stronglyconvex}}

	In order to prove Theorem \ref{theo:stronglyconvex}, we first establish the following two lemmas. The first lemma allows us to establish a descent property for the distance of iterates from the optimal point, while the second lemma bounds the consensus error among the agents. The proofs of these lemmas are given in Appendix~\ref{appen:lemma}.
	
	\begin{lemma}[Decent Lemma]\label{lem:dec}
		Let Assumption \ref{a:strong} hold. Then,
		\begin{align*}
			r_{t+1}\le (1-\mu \eta_t)r_t-\eta_te_t+\frac{3\bar{\sigma}^2}{n}\eta_t^2+2L\eta_tV_t.
		\end{align*}
	\end{lemma}
	
	\begin{lemma}[Consensus Error Lemma]\label{lem:consensus}
		Let Assumption \ref{a:strong} hold. Then,
		\begin{align}\nonumber
			V_t\le H_{k(t)+1}\sum_{j=\tau_{k(t)}}^{t-1}\eta_j^2(12Le_j+6\bar{\sigma}^2).
		\end{align}
		where $k(t)$ is the index such that $\tau_{k(t)}\le t<\tau_{k(t)+1}$.
	\end{lemma}
	
	Using Lemma~\ref{lem:dec} and Lemma~\ref{lem:consensus}, we can now prove Theorem \ref{theo:stronglyconvex}.
	
	\begin{proof}[{\bf Proof of Theorem \ref{theo:stronglyconvex}}]
		For $\eta_t=\frac{2}{\mu(\beta+t)}$, it is easy to see that $(t+\beta)^2(1-
		\mu\eta_t) = (t+\beta)(t+\beta-2)\le (t+\beta-1)^2$. Thus, if we multiply both sides of the expression in Lemma \ref{lem:dec} by $(t+\beta)^2$, we can write 
		\begin{align*}
			(t+\beta)^2r_{t+1}&\le (t+\beta-1)^2r_t-(t+\beta)^2\eta_te_t+(t+\beta)^2\frac{3\bar{\sigma}^2}{n}\eta_t^2+(t+\beta)^22L\eta_tV_t\\
			&=(t+\beta-1)^2r_t-\frac{2(t+\beta)}{\mu}e_t+\frac{12\bar{\sigma}^2}{n\mu^2}+\frac{4L(t+\beta)}{\mu}V_t.
		\end{align*}
		Summing this relation over $t=0,\ldots,T-1$, we get
		\begin{align}\label{eq:3}
			\!\!\!\!(T+\beta-1)^2 r_T&\le (\beta-1)^2r_0-\sum_{t=0}^{T-1}\frac{2(t+\beta)}{\mu}e_t+\frac{12\bar{\sigma}^2T}{n\mu^2}+\frac{4L}{\mu}\sum_{t=0}^{T-1}(t+\beta)V_t.
		\end{align}
		Next, we use Lemma \ref{lem:consensus} to bound the last term $\sum_{t=0}^{T-1}(t+\beta)V_t$ in the above expression~\eqref{eq:3}. We have
		\begin{align}\label{eq:bound-on-V}
			\sum_{t=0}^{T-1}(t+\beta)V_t&\le \sum_{t=0}^{T-1}(t+\beta)H_{k(t)+1}\!\!\sum_{j=\tau_{k(t)}}^{t-1}\!\eta_j^2(12Le_j+6\bar{\sigma}^2)\cr 
			&= \sum_{j = 0}^{T-2}\eta_j^2(12Le_j+6\bar{\sigma}^2)\sum_{t=j+1}^{\tau_{k(j)+1}-1}(t+\beta)H_{k(t)+1}\cr 
			&= \sum_{j = 0}^{T-2}\eta_j^2(12Le_j+6\bar{\sigma}^2)H_{k(j)+1}\!\!\!\sum_{t=j+1}^{\tau_{k(j)+1}-1}\!\!(t+\beta),  
		\end{align}
		where the last equality holds because $k(t)=k(j)$ for any $t\in [j+1, \tau_{k(j)+1}-1]$. Moreover, using the assumption on the communication intervals, we have
		\begin{align*}
			\tau_{k(j)+1}-\tau_{k(j)}=H_{k(j)+1}\leq \beta+\sum_{\ell=1}^{k(j)}H_{\ell}=\beta+\tau_{k(j)},
		\end{align*}
		which implies $\tau_{k(j)+1}\leq \beta+2\tau_{k(j)}$. Using this relation together with $\tau_{k(j)}\le j<\tau_{k(j)+1}$, we can write
		\begin{align}\nonumber
			\sum_{t=j+1}^{\tau_{k(j)+1}-1}(t+\beta)&\leq \sum_{t=\tau_{k(j)}}^{\tau_{k(j)+1}-1}(t+\beta)\cr 
			&\leq (\tau_{k(j)+1}-\tau_{k(j)})(\beta+\frac{\tau_{k(j)+1}+\tau_{k(j)}}{2})\cr
			&\leq H_{k(j)+1}(\beta+\frac{\beta+3\tau_{k(j)}}{2})\leq \frac{3}{2} H_{k(j)+1}(j+\beta).
		\end{align} 
		Substituting this relation into~\eqref{eq:bound-on-V}, we get
		\begin{align}\nonumber
			\sum_{t=0}^{T-1}(t+\beta)V_t&\leq \sum_{j = 0}^{T-2}\eta_j^2(18Le_j+9\bar{\sigma}^2)(j+\beta)H^2_{k(j)+1}\cr 
			&=\sum_{t = 0}^{T-2}(\frac{72L}{\mu^2}e_t+\frac{36}{\mu^2}\bar{\sigma}^2)\frac{H^2_{k(t)+1}}{t+\beta}.
		\end{align}
		where in the second equality we have used $\eta_t=\frac{2}{\mu(\beta+t)}$ and relabeled the index $j$ by $t$. Finally, if we substitute the above relation into~\eqref{eq:3}, we obtain 
		\begin{align}\label{eq:error-term-strongly-convex}
			\!(T\!+\!\beta\!-\!1)^2 r_T\!-\!(\beta\!\!-1\!)^2r_0\!\le\! \frac{12\bar{\sigma}^2T}{n\mu^2}+\sum_{t=0}^{T-1}(\frac{288L^2}{\mu^3}\frac{H^2_{k(t)+1}}{t\!+\!\beta}-\frac{2(t+\beta)}{\mu})e_t\!+\!\frac{144L\bar{\sigma}^2}{\mu^3}\sum_{t = 0}^{T-2}\frac{H^2_{k(t)+1}}{t+\beta}.
		\end{align}   
		Now, using the condition on the length of communication intervals in the theorem statement, we know that
		\begin{align}\nonumber
			H_{k(t)+1}\leq \frac{\mu (\beta + \sum_{j = 1}^{k(t)}H_j)}{12L}=\frac{\mu (\beta+\tau_{k(t)})}{12L}\leq \frac{\mu (\beta+t)}{12L}.  
		\end{align}
		Substituting this bound in~\eqref{eq:error-term-strongly-convex} we obtain
		\begin{align}\nonumber
			(T+\beta-1)^2 r_T-(\beta-1)^2r_0
			&\le \frac{12\bar{\sigma}^2T}{n\mu^2}+\frac{144L\bar{\sigma}^2}{\mu^3}\sum_{t = 0}^{T-2}\frac{H^2_{k(t)+1}}{t+\beta}\cr
			&=\frac{12\bar{\sigma}^2T}{n\mu^2}+\frac{144L\bar{\sigma}^2}{\mu^3}\sum_{i=1}^{R}\sum_{t = \tau_{i-1}}^{\tau_{i}-1}\frac{H^2_{k(t)+1}}{t+\beta}\cr
			&\le\frac{12\bar{\sigma}^2T}{n\mu^2}+\frac{144L\bar{\sigma}^2}{\mu^3}\sum_{i=1}^{R}\frac{H^3_{i}}{\tau_{i-1}+\beta}\cr
			& =\frac{12\bar{\sigma}^2T}{n\mu^2}+\frac{144L\bar{\sigma}^2}{\mu^3}\sum_{i=1}^{R}\frac{H^3_{i}}{\sum_{j=1}^{i-1}H_j+\beta},
		\end{align}
		where the second equality holds because for any $t\in [\tau_{i-1}, \tau_{i})$, we have $k(t)+1=i$. Dividing both sides by $T^2$, we obtain the desired bound.
	\end{proof}
	
	\subsection{Proof of Theorem \ref{theo:convex}}
	\begin{proof}
		Let us set $\eta_t=\eta, \forall t$, for some parameter $\eta$ to be determined later. Substituting $\mu=0$ in Lemma \ref{lem:dec} and summing over $t=0,\ldots,T-1$, we get
		\begin{align}\label{eq:conv-5}
			\eta\sum_{t=0}^{T-1}e_t\le r_0-r_{T}+\frac{3\bar{\sigma}^2\eta^2T}{n}+2L\eta\sum_{t=0}^{T-1}V_t.
		\end{align}
		Next, we use Lemma \ref{lem:consensus} to bound $\sum_{t=0}^{T-1}V_t$. We have,
		\begin{align}\label{eq:V-bound-convex}
			\sum_{t=0}^{T-1}V_t&\le \sum_{t=0}^{T-1}H_{k(t)+1}\sum_{j=\tau_{k(t)}}^{t-1}\eta^2(12Le_j+6\bar{\sigma}^2)\cr 
			&\leq 12L\eta^2\sum_{j = 0}^{T-2}e_j\sum_{t=j+1}^{\tau_{k(j)+1}-1}H_{k(t)+1}+6\bar{\sigma}^2\eta^2\sum_{t=0}^{T-1}H_{k(t)+1}(t-\tau_{k(t)})\cr
			&\le 12L\eta^2\sum_{j = 0}^{T-2}e_jH_{k(j)+1}^2\!+\!6\bar{\sigma}^2\eta^2\sum_{t=0}^{T-1}H_{k(t)+1}^2\cr 
			&\le\frac{1}{4L}\sum_{t = 0}^{T-2}e_t+\!6\bar{\sigma}^2\eta^2\sum_{i=0}^{R}H_{i}^3,
		\end{align}
		where in the third inequality we have used the fact that $k(t)=k(j)$ for any $t\in[j+1, \tau_{k(j)+1}-1]$, and $t-\tau_{k(t)}\leq H_{k(t)+1}$, and in the last inequality we have used $H_i\le \frac{1}{7L\eta},\forall i$. Now, we can write
		\begin{align*}
			&\frac{1}{2}\sum_{t=0}^{T-1}e_t\le \frac{r_0-r_T}{\eta}+\frac{3\bar{\sigma}^2\eta T}{n}+\!12L\bar{\sigma}^2\eta^2\sum_{i=0}^{R}H_{i}^3.
		\end{align*}
		Dividing both sides of the above inequality by $T$ and using the choice of $\eta=c\sqrt{\frac{n}{T}}$, we obtain
		\begin{align*}
			\frac{1}{T}\sum_{t=0}^{T-1}e_t\le \frac{2r_0+6c^2\bar{\sigma}^2}{c\sqrt{nT}}+\frac{24L\bar{\sigma}^2c^2n}{T^2}\sum_{i = 1}^RH_i^3.
		\end{align*}
		
	\end{proof}
	
	\subsection{Proof of Theorem \ref{theo:nonconvex}}
	To prove Theorem \ref{theo:nonconvex}, we first establish an analogous descent lemma and consensus error lemma for the case of nonconvex local functions. The proofs of these lemmas are given in Appendix~\ref{appen:lemma}.
	\begin{lemma}[Decent Lemma, Non-Convex]\label{lem:dec_nonconvex}
		Assume that $\nabla F_i(\xx,\xi_i)$ is an unbiased stochastic gradient of $f_i(\xx)$ with variance bounded by $\sigma^2$. We have
		\begin{align}\nonumber
			e_{t+1}\le e_t-\frac{\eta_t}{4}h_t+\frac{L\sigma^2}{2n}\eta_t^2+L^2\eta_tV_t.
		\end{align}
	\end{lemma}
	
	\begin{lemma}[Consensus Error Lemma, Non-Convex]\label{lem:consensus_nonconvex}
		Let Assumption \ref{a:nonconvex} hold. Moreover, assume that $\nabla F_i(\xx,\xi_i)$ is an unbiased stochastic gradient of $f_i(\xx)$ with variance bounded by $\sigma^2$. For any $t$, define $k(t)$ be the index such that $\tau_{k(t)}\le t<\tau_{k(t)+1}$. We have
		\begin{align}\nonumber
			V_t\le H_{k(t)+1}\sum_{j=\tau_{k(t)}}^{t-1}6\eta_j^2(B^2h_j+\sigma^2+G^2).
		\end{align}
	\end{lemma}
	
	\begin{proof}[{\bf Proof of Theorem \ref{theo:nonconvex}}]
		Let us choose $\eta_t=\eta$, for some $\eta$ to be specified later. By summing Lemma \ref{lem:dec_nonconvex} over $t=0,\ldots,T-1$, we get
		\begin{align}
			\frac{\eta}{4}\sum_{t=0}^{T-1}h_t\le e_0-e_{T}+\frac{L\sigma^2\eta^2T}{2n}+L^2\eta\sum_{t=0}^{T-1}V_t.
			\label{eq:5_nonconvex}
		\end{align}
		Next, we use Lemma \ref{lem:consensus_nonconvex} to bound $\sum_{t=0}^{T-1}V_t$. Using the same idea as in deriving expression~\eqref{eq:V-bound-convex} in the proof of Theorem \ref{theo:convex}, we can get
		\begin{align*}
			\sum_{t=0}^{T-1}V_t&\le \sum_{t=0}^{T-1}H_{k(t)+1}\sum_{j=\tau_{k(t)}}^{t-1}6\eta^2(B^2h_j+\sigma^2+G^2)\cr 
			&\leq 6B^2\eta^2\sum_{j = 0}^{T-2}e_j\sum_{t=j+1}^{\tau_{k(j)+1}-1}H_{k(t)+1}+6(\sigma^2+G^2)\eta^2\sum_{t=0}^{T-1}H_{k(t)+1}(t-\tau_{k(t)})\cr
			&\le 6B^2\eta^2\sum_{j = 0}^{T-2}e_jH_{k(j)+1}^2+6(\sigma^2+G^2)\eta^2\sum_{t=0}^{T-1}H_{k(t)+1}^2\cr 
			&\le\frac{1}{8L^2}\sum_{t = 0}^{T-2}h_t+\!6(\sigma^2+G^2)\eta^2\sum_{i=0}^{R}H_{i}^3,
		\end{align*}
		where in the last inequality we have used $H_i\le \frac{1}{7BL\eta},\forall i$. Now, we can write
		\begin{align*}
			\frac{1}{8}\sum_{t=0}^{T-1}h_t\le \frac{e_0-e_T}{\eta}+\frac{L\sigma^2\eta T}{2n}+6L^2(\sigma^2\!+\!G^2)\eta^2\!\sum_{i=0}^{R}H_{i}^3.
		\end{align*}
		Substituting $\eta_t =c\sqrt{\frac{n}{T}}$ into the above inequality and dividing both sides by $T$ we get the desired bound.
	\end{proof}
	
	\subsection{Proof of Lemmas}\label{appen:lemma}
	\begin{proof}[{\bf Proof of Lemma \ref{lem:dec}}] 
		Consider the filtration $\{\F^t\}_{t=1}^{\infty}$ adapted to the history of random variables $\{\xi_i^{(t)}\}$, i.e., 
		\begin{align}\nonumber
			\F^t=\{\xi_i^{(k)} | i\in [n], 0\leq k \leq t-1 \},
		\end{align}
		and note that $\bg_i^{t-1}=\nabla F_i(\bx_i^{t-1},\xi_i^{(t-1)})$ and $\bx_i^{t}$ are $\F^t$-measurable, but $\bg_i^t$ is not.  Using the definition of $r_t= \E\|\bbx^{(t)} -\bx^*\|^2$ and $\bbg^{(t)}=\frac{1}{n}\sum_{i=1}^n\bg_i^t$, we have
		\begin{align*}
			r_{t+1}&= \E[\|\bbx^{(t+1)} -\bx^*\|^2]\cr 
			&=\E[\|\bbx^{(t)} -\eta_t\bbg^{(t)}-\bx^*\|^2]\cr
			&=\E\big[\big\|\bbx^{(t)}-\bx^*-\frac{\eta_t}{n}\sum_{i=1}^n\nabla f_i(\bx_i^{t})-\frac{\eta_t}{n}\sum_{i=1}^n(\bbg^{(t)}-\nabla f_i(\bx_i^{t}))\big\|^2\big]\cr 
			&=\E\big[\big\|\bbx^{(t)} -\bx^*-\frac{\eta_t}{n}\sum_{i=1}^n\nabla f_i(\bx_i^{t})\big\|^2\big]\eta_t^2\E\big[\big\|\frac{1}{n}\sum_{i=1}^n(\bbg^{(t)}-\nabla f_i(\bx_i^{t}))\big\|^2\big]\cr
			&-\frac{2\eta_t}{n}\mathbb{E}_{\mathcal{F}^t}\big[\big\langle \bbx^{(t)}-\bx^*-\frac{\eta_t}{n}\sum_{i=1}^n\!\nabla f_i(\bx_i^{t}), \mathbb{E}\big[\sum_{i=1}^n(\bbg^{(t)}-\nabla f_i(\bx_i^{t}))\big|\mathcal{F}^{t}\big]\big\rangle\big],
		\end{align*}
		where the last term is obtained by first conditioning on $\mathcal{F}^t$ and then taking expectation with respect to $\mathcal{F}^t$. However, the last inner product in the above expression is zero, we have $\E[\bg_i^{t} |\mathcal{F}^{t}] = \nabla f_i(\bx_i^{t}), \forall i$, and thus $\mathbb{E}[\sum_{i=1}^n(\bbg^{(t)}\!-\!\nabla f_i(\bx_i^{t}))|\mathcal{F}^{t}]=0$.
		Therefore, we have
		\begin{align}\label{eq:E-r-t}
			r_{t+1}=\E\big[\big\|\bbx^{(t)} -\bx^*-\frac{\eta_t}{n}\sum_{i=1}^n\nabla f_i(\bx_i^{t})\big\|^2\big]+\eta_t^2\E\big[\big\|\frac{1}{n}\sum_{i=1}^n(\bbg^{(t)}-\nabla f_i(\bx_i^{t}))\big\|^2\big].
		\end{align}
		We can bound the second term in equation~\eqref{eq:E-r-t} using~\cite[Proposition 5]{koloskova2020unified} as the following: 
		\begin{align}\nonumber
			\mathbb{E}\big[\big\|\frac{1}{n}\sum_{i=1}^n(\bbg^{(t)}-\nabla f_i(\bx_i^{t}))\big\|^2\big]
			&\leq \frac{3L^2}{n^2} \!\sum_{i = 1}^{n}\mathbb{E}\|\xx_i - \bar{\xx}\|^2 \!+\! \frac{6L}{n}\left(\mathbb{E}[f(\bar{\xx})] \!-\! f(\xx^\star)\right) \!+\! \frac{3 \bar{\sigma}^2}{n}\cr 
			&=\frac{3L^2}{n^2} V_t+ \frac{6L}{n}e_t+ \frac{3 \bar{\sigma}^2}{n}.
		\end{align}
		In order to bound the first term in~\eqref{eq:E-r-t}, we can write
		\begin{align}\label{eq:first-term-descent}
			&\E\|\bbx^{(t)} -\bx^*-\frac{\eta_t}{n}\sum_{i=1}^n\nabla f_i(\bx_i^{t})\|^2\cr=&\E\|\bbx^{(t)} -\bx^*\|^2+\frac{\eta_t^2}{n^2}\E\|\sum_{i=1}^n\nabla f_i(\bx_i^{t})\|^2-\frac{2\eta_t}{n}\E \langle \bbx^{(t)} -\bx^*,\sum_{i=1}^n\nabla f_i(\bx_i^{t})\rangle.
		\end{align}
		To bound $\E\|\sum_{i=1}^n\nabla f_i(\bx_i^{t})\|^2$ in~\eqref{eq:first-term-descent}, we can write
		\begin{align}\nonumber
			\E\|\sum_{i=1}^n\nabla f_i(\bx_i^{t})\|^2 &\!\le\! 2\E\|\!\sum_{i=1}^n\!\nabla f_i(\bx_i^{t})\!-\!\sum_{i=1}^n\nabla f_i(\bbx^{t})\|^2+2\E\|\!\sum_{i=1}^n\!\nabla f_i(\bbx^{t})\!-\!\sum_{i=1}^n\!\nabla f_i(\bx_*^{t})\|^2\cr
			&\le 2n\E\sum_{i = 1}^n\|\nabla f_i(\bx_i^{t})-\nabla f_i(\bbx^{t})\|^2+2n\E\sum_{i = 1}^n\|\nabla f_i(\bbx^{t})-\nabla f_i(\bx_*^{t})\|^2\cr
			&\le 2nL^2 \E\sum_{i=1}^n\|\xx_i^{(t)}-\bbx^{(t)}\|^2+4n^2L(\E f(\bbx^{(t)})-f(\bx^*))\cr 
			&=2n^2L^2 V_t+4n^2Le_t,
		\end{align}
		where in the last inequality, we have used Lemma \ref{prop:comvex,smooth}. To bound $\E \langle \bbx^{(t)} -\bx^*,\sum_{i=1}^n\nabla f_i(\bx_i^{t})\rangle$ in~\eqref{eq:first-term-descent}, we have
		\begin{align}\nonumber
			&\E\langle \bbx^{(t)}\!-\!\bx^*,\sum_{i=1}^n\!\nabla f_i(\bx_i^{t})\rangle\!\cr=&\!\sum_{i=1}^n \E\langle \bbx^{(t)} \!-\!\bx^*,\nabla f_i(\bx_i^{t})\rangle \cr 
			=&\sum_{i=1}^n \E\langle \bbx^{(t)} -\bx_i^{(t)},\nabla f_i(\bx_i^{t})\rangle+\sum_{i=1}^n \E\langle \bx_i^{(t)} -\bx^*,\nabla f_i(\bx_i^{t})\rangle\cr   
			\ge&\sum_{i = 1}^n\E\big[f_i(\bbx^{(t)})-f_i(\bx_i^{(t)})-\frac{L}{2}\|\bbx^{(t)}-\bx_i^{(t)}\|^2\big]+\sum_{i = 1}^n\E\big[f_i(\bx_i^{(t)})-f_i(\bx^{*})+\frac{\mu}{2}\|\bx_i^{(t)}-\bx^{*}\|^2\big]\cr
			=& \sum_{i = 1}^n \E[f_i(\bbx^{(t)})-f_i(\bx^*)]+\frac{\mu n}{2}\E\|\bbx^{(t)}-\bx^{*}\|^2-(\frac{L-\mu}{2})\sum_{i = 1}^n\E\|\bx_i^{(t)}-\bbx^{(t)}\|^2\cr 
			=&ne_t+\frac{\mu}{2}n r_t-(\frac{L-\mu}{2})nV_t,
		\end{align}
		where the first inequality follows from the strong convexity assumption and Lemma \ref{prop:comvex,smooth}. Moreover, in the last equality, we used Lemma \ref{eq:algebra}. Finally, if we put the above bounds into~\eqref{eq:first-term-descent} and substitute the result into~\eqref{eq:E-r-t}, we obtain
		\begin{align*}
			r_{t+1}&\!\le\! r_t\!+\!\frac{\eta_t^2}{n^2}(2n^2L^2V_t\!+\!4n^2Le_t)\!-\!\frac{2\eta_t}{n}(ne_t\!+\!\frac{\mu}{2}nr_t\!-\!\frac{L-\mu}{2}nV_t)+\eta_t^2(\frac{3L^2}{n}V_t+\frac{6L}{n}e_t+\frac{3\bar{\sigma}^2}{n})\\
			&=(1-\mu\eta_t )r_t-(2\eta_t-4\eta_t^2L-\frac{6L\eta_t^2}{n})e_t+(2L^2\eta_t^2+(L-\mu)\eta_t+\frac{3L^2\eta_t^2}{n})V_t+\frac{3\bar\sigma^2\eta_t^2}{n}\\
			&\le (1-\mu\eta_t )r_t-\eta_te_t+2L\eta_tV_t+\frac{3\bar\sigma^2\eta_t^2}{n},
		\end{align*}
		where the last inequality holds because $\eta_t\le \frac{1}{10L}$.
	\end{proof}
	
	In order to prove the consensus error lemma (Lemma \ref{lem:consensus}), we first state and prove the following auxiliary lemma, which bounds the expected sum of the gradient norms across all agents. 
	
	\begin{lemma}\label{eq:gradnorm}
		For strongly convex $L$-smooth local functions, we have
		\begin{align}\nonumber
			\frac{1}{n}\sum_{i = 1}^n\E\|\nabla F_i(\xx_i^{(t)}\!, \xi_i^{(t)})\|^2\le 3L^2V_t+6Le_t+3\bar{\sigma}^2.
		\end{align}
	\end{lemma}
	\begin{proof}
		Starting from the left-hand side, we can write
		\begin{align}\label{lemm:intermediate-consensus}
			&\frac{1}{n}\sum_{i = 1}^n\E\|\nabla F_i(\xx_i^{(t)}, \xi_i^{(t)})\|^2\cr
			=&\frac{1}{n}\sum_{i = 1}^n\E\big\|\nabla F_i(\xx_i^{(t)}\!, \xi_i^{(t)})-\nabla F_i(\bbx^{(t)}\!, \xi_i^{(t)})+\!\nabla F_i(\bbx^{(t)}\!, \xi_i^{(t)})\!-\!\nabla F_i(\xx^*\!, \xi_i^{(t)})\!+\!\nabla F_i(\xx^*\!, \xi_i^{(t)})\big\|^2\cr
			\le& \frac{3}{n}\sum_{i = 1}^n\E\|\nabla F_i(\bbx^{(t)}\!, \xi_i^{(t)})-\nabla F_i(\xx^*\!, \xi_i^{(t)})\|^2+\frac{3}{n}\sum_{i = 1}^n\E\|\nabla F_i(\xx_i^{(t)}\!, \xi_i^{(t)})-\nabla F_i(\bbx^{(t)}\!, \xi_i^{(t)})\|^2+3\bar{\sigma}^2\cr
			\le&\frac{3}{n}\sum_{i = 1}^n\E\|\nabla F_i(\bbx^{(t)}\!, \xi_i^{(t)})-\nabla F_i(\xx^*\!, \xi_i^{(t)})\|^2+3L^2V_t+3\bar{\sigma}^2,
		\end{align}
		where the first inequality uses Definition \ref{eq:var_opt}, and the second inequality uses $L$-smooth assumption. We have
		\begin{align*}
			&\sum_{i = 1}^n\E\|\nabla F_i(\bbx^{(t)}\!, \xi_i^{(t)})-\nabla F_i(\bx^*\!, \xi_i^{(t)})\|^2\cr
			\le& 2L\sum_{i = 1}^n\E\big[F_i(\bbx^{(t)}\!, \xi_i^{(t)})-F_i(\bx^*\!, \xi_i^{(t)})\big]-2L\sum_{i = 1}^n\E\langle\nabla F_i(\bx^*\!, \xi_i^{(t)}),\bbx^{(t)}-\bx^*\rangle\cr
			=&2L\sum_{i = 1}^n\E\big[f_i(\bbx^{(t)})-f_i(\bx^*)\big]-2L \sum_{i = 1}^n\E\langle\nabla f_i(\bx^*),\bbx^{(t)}-\bx^*\rangle\cr 
			=&2nL\E\big[f(\bbx^{(t)})-f(\bx^*)-\langle\nabla f(\bx^*),\bbx^{(t)}-\bx^*\rangle\big]\cr
			=& 2nL\E\big[f(\bbx^{(t)})-f(\bx^*)\big]= 2nLe_t,
		\end{align*}
		where the first inequality uses Lemma \ref{prop:comvex,smooth}, and the last equality holds because $\bx^*$ is the global minimum of $f$, and hence $\nabla f(\bx^*)=0$. Substituting the above relation into~\eqref{lemm:intermediate-consensus} completes the proof. 
	\end{proof}

	\begin{proof}[{\bf Proof of Lemma \ref{lem:consensus}}]
		As $\tau_{k(t)}\leq t< \tau_{k(t)+1}$, agents do not communicate during the time interval $(\tau_{k(t)}, t]$ and only perform local gradient steps. Thus, we have 
		\begin{align}\nonumber
			\xx_i^{(t)}=\xx_i^{(\tau_{k(t)})}-\sum_{j=\tau_{k(t)}}^{t-1}\eta_j\nabla F_i(\bx_i^{(j)},\xi_i^{(j)}). 
		\end{align}
		Moreover, at the communication time $\tau_{k(t)}$, all the agents update their local vectors to the same average vector received from the center node. Therefore, $\bbx^{\tau_{k(t)}}=\xx_i^{(\tau_{k(t)})} \forall i$, and we have
		\begin{align*}
			&\bbx^{t}=\bbx^{\tau_{k(t)}}-\!\!\!\sum_{j=\tau_{k(t)}}^{t-1}\!\!\eta_j\bbg^{(j)}=\xx_i^{(\tau_{k(t)})}-\!\!\!\sum_{j=\tau_{k(t)}}^{t-1}\!\!\eta_j\bbg^{(j)}.
		\end{align*}
		If we substitute the above relations into $V_t$, we get 
		\begin{align*}
			nV_t &= \E\sum_{i = 1}^n\|\xx_i^{(t)}-\bbx^{t}\|^2\cr 
			&=\sum_{i = 1}^n\E\|\!\!\sum_{j=\tau_{k(t)}}^{t-1}\!\!\eta_j\nabla F_i(\bx_i^{(j)},\xi_i^{(j)})-\!\!\sum_{j=\tau_{k(t)}}^{t-1}\!\!\eta_j\bbg^{(j)}\|^2\cr 
			&\le\sum_{i = 1}^n\E\|\sum_{j=\tau_{k(t)}}^{t-1}\eta_j\nabla F_i(\bx_i^{(j)},\xi_i^{(j)})\|^2\\
			&\le (t-\tau_{k(t)}) \sum_{i = 1}^n\sum_{j=\tau_{k(t)}}^{t-1}\eta_j^2\E\|\nabla F_i(\bx_i^{(j)},\xi_i^{(j)})\|^2\cr 
			&\le n(t-\tau_{k(t)})\sum_{j=\tau_{k(t)}}^{t-1}\eta_j^2(3L^2V_j+6Le_j+3\bar{\sigma}^2)\\
			&\le nH_{k(t)+1}\sum_{j=\tau_{k(t)}}^{t-1}\eta_j^2(3L^2V_j+6Le_j+3\bar{\sigma}^2).
		\end{align*}
		where the first inequality uses Lemma \ref{eq:algebra}, and the third inequality follows from Lemma \ref{eq:gradnorm}. Moreover, by our choice of step-size $\eta_j\le \frac{1}{3LH_{k(t)+1}}, \forall j\ge \tau_{k(t)}$. Thus, for any time instance in the interval $[\tau_{k(t)}, \tau_{k(t)+1})$, we have shown that 
		\begin{align}\nonumber
			V_t\!\le\! \frac{1}{3H_{k(t)+1}}\!\sum_{j=\tau_{k(t)}}^{t-1}\!\!V_j\!+\!H_{k(t)+1}\!\!\!\sum_{j=\tau_{k(t)}}^{t-1}\!\!\eta_j^2(6Le_j\!+\!3\bar{\sigma}^2).
		\end{align}
		By recursively unrolling $V_j, j=\tau_{k(t)},\ldots,t-1$, and noting that $V_{\tau_{k(t)}}=0$, we obtain 
		\begin{align*}
			&V_t\le \frac{1}{3H_{k(t)+1}}V_{t-1}+\frac{1}{3H_{k(t)+1}}\sum_{j=\tau_{k(t)}}^{t-2}V_j+H_{k(t)+1}\sum_{j=\tau_{k(t)}}^{t-1}\eta_j^2(6Le_j+3\bar{\sigma}^2)\cr 
			\leq& \big((\frac{1}{3H_{k(t)+1}})^2+\frac{1}{3H_{k(t)+1}}\big)\!\sum_{j=\tau_{k(t)}}^{t-2}\!\!V_j+(1+\frac{1}{3H_{k(t)+1}})H_{k(t)+1}\!\!\!\sum_{j=\tau_{k(t)}}^{t-1}\!\!\eta_j^2(6Le_j\!+\!3\bar{\sigma}^2)\le \cdots\cr  
			\le& (1\!+\!\frac{1}{3H_{k(t)+1}})^{t-\tau_{k(t)}}H_{k(t)+1}\!\!\!\!\sum_{j=\tau_{k(t)}}^{t-1}\!\!\!\eta_j^2(6Le_j\!+\!3\bar{\sigma}^2)\cr 
			\le& (1\!+\!\frac{1}{3H_{k(t)+1}})^{H_{k(t)+1}}H_{k(t)+1}\!\!\!\!\sum_{j=\tau_{k(t)}}^{t-1}\!\!\!\eta_j^2(6Le_j\!+\!3\bar{\sigma}^2).
		\end{align*}
		Finally, by replacing $(1+\frac{1}{3H_{k(t)+1}})^{H_{k(t)+1}}\leq 2$ into the above relation we obtain the desired bound.  
	\end{proof}
	
	\begin{proof}[{\bf Proof of Lemma \ref{lem:dec_nonconvex}}]Using Taylor expansion and the $L$-smoothness assumption, we can write 
		\begin{align}\label{eq:descent-nonconvex}
			e_{t+1}&=\E f(\bbx^{(t+1)})-f(\bx^*)\cr 
			&=\E f\big(\bbx^{(t)}-\eta_t\bbg^{(t)})\big)-f(\bx^*)\cr
			&\le \big(\E f(\bbx^{(t)})\!-\!f(\bx^*)\big)+\frac{L}{2}\eta_t^2\E \|\bbg^{(t)}\|^2-\eta_t \E \langle\nabla f(\bbx^{(t)}),\bbg^{(t)})\rangle,
		\end{align}
		where we recall that $\bbg^{(t)}=\frac{1}{n}\sum_{i = 1}^n\nabla F_i(\bx_i^{(t)},\xi_i^{(t)})$. 
		Next, we bound the second and third terms in~\eqref{eq:descent-nonconvex}. To bound the third term, using Assumption \ref{a:nonconvex}, we have
		\begin{align}\label{eq:nonconvex-descent-third}
			&\E \langle\nabla f(\bbx^{(t)}),\bbg^{(t)})\rangle\cr 
			=&\E \langle\nabla f(\bbx^{(t)}),\frac{1}{n}\!\sum_{i = 1}^n\!\nabla F_i(\bx_i^{(t)},\xi_i^{(t)})\rangle\cr 
			=&\E \big[\mathbb{E}_{\{\xi_i^{(t)}\}}\big[\langle\nabla f(\bbx^{(t)}),\frac{1}{n}\!\sum_{i = 1}^n\!\nabla F_i(\bx_i^{(t)},\xi_i^{(t)})\big\rangle|\mathcal{F}^t\big]\big]\cr 
			=&\E \big[\langle\nabla f(\bbx^{(t)}),\frac{1}{n}\sum_{i = 1}^n \mathbb{E}_{\xi_i^{(t)}}[\nabla F_i(\bx_i^{(t)},\xi_i^{(t)})|\mathcal{F}^t]\rangle\big]\cr 
			=&\E \langle\nabla f(\bbx^{(t)}),\frac{1}{n}\sum_{i = 1}^n\nabla f_i(\bx_i^{(t)})\rangle\cr  
			=&\E \langle\nabla f(\bbx^{(t)}),\frac{1}{n}\sum_{i=1}^n(\nabla f_i(\bx_i^{(t)})-\nabla f_i(\bbx^{(t)})\rangle+\E \|\nabla f(\bbx^{(t)})\|^2\cr
			\ge& \frac{1}{2}\E \|\nabla f(\bbx^{(t)})\|^2-\frac{1}{2n}\sum_{i = 1}^n\E\|\nabla f_i(\bx_i^{(t)})-\nabla f_i(\bbx^{(t)})\|^2\cr
			\ge&\frac{1}{2}\E \|\nabla f(\bbx^{(t)})\|^2-\frac{L^2}{2n}\sum_{i = 1}^n\E\|\bbx^{(t)}-\bx_i^{(t)}\|^2\cr
			=&\frac{1}{2}h_t-\frac{L^2}{2}V_t,
		\end{align}
		where the last inequality is by $L$-smoothness assumption. To bound the second term in~\eqref{eq:descent-nonconvex}, we have
		\begin{align*}
			&\E\|\bbg^{(t)}\|^2=\E\|\frac{1}{n}\sum_{i = 1}^n\nabla F_i(\bx_i^{(t)},\xi_i^{(t)})\|^2\cr 
			=&\E\|\frac{1}{n}\sum_{i = 1}^n(\nabla F_i(\bx_i^{(t)},\xi_i^{(t)})-\nabla f_i(\bx_i^{(t)})) \|^2+\E\|\frac{1}{n}\sum_{i = 1}^n\nabla f_i(\bx_i^{(t)})\|^2\\
			=&\frac{1}{n^2}\sum_{i = 1}^n\E\|\nabla F_i(\bx_i^{(t)},\xi_i^{(t)})-\nabla f_i(\bx_i^{(t)})\|^2+\E\|\frac{1}{n}\sum_{i = 1}^n(\nabla f_i(\bx_i^{(t)})-\nabla f_i(\bbx^{(t)}))+\nabla f(\bbx^{(t))})\|^2.
		\end{align*}	
		Thus, using the bounded noise Assumption \ref{a:nonconvex}, we get
		\begin{align}\label{eq:nonconvex-descent-second}
			\E\|\bbg^{(t)}\|^2\le& \frac{\sigma^2}{n}+2\E \|\nabla f(\bbx^{(t)})\|^2+\frac{2}{n}\sum_{i = 1}^n\E\|\nabla f_i(\bx_i^{(t)})-\nabla f_i(\bbx^{(t)})\|^2\cr
			\le& \frac{\sigma^2}{n}+2\E \|\nabla f(\bbx^{(t)})\|^2+\frac{2L^2}{n}\sum_{i = 1}^n\E\|\bbx^{(t)}-\bx_i^{(t)}\|^2\cr
			=&\frac{\sigma^2}{n}+2h_t+2L^2V_t,
		\end{align}
		where the second inequality holds by the $L$-smooth assumption. Finally, by substituting~\eqref{eq:nonconvex-descent-third} and~\eqref{eq:nonconvex-descent-second} into~\eqref{eq:descent-nonconvex}, we obtain
		\begin{align*}
			e_{t+1}&\le e_t-\eta_t(\frac{1}{2}h_t-\frac{L^2}{2}V_t)+\frac{L\eta_t^2}{2}(\frac{\sigma^2}{n}+2h_t+2L^2V_t)\\
			&\le e_t-\frac{\eta_t}{4}h_t+\frac{L\sigma^2}{2n}\eta_t^2+L^2\eta_tV_t,
		\end{align*}
		where the last inequality holds because $\eta_t\le \frac{1}{4L}$. 
	\end{proof}
	
	We first establish the following technical lemma to prove the consensus descent lemma for the nonconvex functions.
	
	\begin{lemma}\label{lemm:intermediate-nonconvex}
		Let Assumptions \ref{a:nonconvex} hold. Then,
		\begin{align}\nonumber
			\frac{1}{n}\!\sum_{i = 1}^n\E\|\nabla F_i(\xx_i^{(t)}\!, \xi_i^{(t)})\|^2\!\le\! 3(L^2V_t\!+\!B^2h_t\!+\!\sigma^2\!+\!G^2).
		\end{align}
	\end{lemma}
	\begin{proof}
		We can write,
		\begin{align*}
			&\frac{1}{n}\sum_{i = 1}^n\E\|\nabla F_i(\xx_i^{(t)}\!, \xi_i^{(t)})\|^2\\
			=&\frac{1}{n}\sum_{i = 1}^n\E\big\|\nabla F_i(\xx_i^{(t)}\!, \xi_i^{(t)})-\nabla f_i(\bx_i^{(t)})+\nabla f_i(\bx_i^{(t)})-\nabla f_i(\bbx^{(t)})+\nabla f_i(\bbx^{(t)})\big\|^2\\
			\le& \frac{3}{n}\sum_{i = 1}^n\E\|\nabla F_i(\xx_i^{(t)}\!, \xi_i^{(t)})\!-\!\nabla f_i(\bx_i^{(t)})\|^2\!+\!\frac{3}{n}\sum_{i = 1}^n\E\|\nabla f_i(\bx_i^{(t)})\!-\!\nabla f_i(\bbx^{(t)})\|^2\!+\!\frac{3}{n}\sum_{i = 1}^n\E\|\nabla f_i(\bbx^{(t)})\|^2\\
			\le& 3\sigma^2+3L^2V_t+3(G^2+B^2h_t).
		\end{align*}
		where the last inequality is obtained using the $L$-smoothness assumption and Assumption \ref{a:nonconvex}.   
	\end{proof}

	\begin{proof}[{\bf Proof of Lemma \ref{lem:consensus_nonconvex}}]
		By following the same steps as in the proof of Lemma \ref{lem:consensus}, we can write
		\begin{align*}
			nV_t &= \E\sum_{i = 1}^n\|\xx_i^{(t)}-\bbx^{t}\|^2\cr
			& \le (t-\tau_{k(t)}) \sum_{i = 1}^n\sum_{j=\tau_{k(t)}}^{t-1}\eta_j^2\E\|\nabla F_i(\bx_i^{(j)},\xi_i^{(j)})\|^2\\
			&\le n(t-\tau_{k(t)})\sum_{j=\tau_{k(t)}}^{t-1}3\eta_j^2\big(L^2V_j+B^2h_j+\sigma^2+G^2\big)\\
			&\le nH_{k(t)+1}\sum_{j=\tau_{k(t)}}^{t-1}3\eta_j^2\big(L^2V_j+B^2h_j+\sigma^2+G^2\big)\cr 
			&\le nH_{k(t)+1}\sum_{j=\tau_{k(t)}}^{t-1}3\eta_j^2\big(L^2V_j+B^2h_j+\sigma^2+G^2\big),
		\end{align*}
		where in the second inequality, we have used Lemma \ref{lemm:intermediate-nonconvex}. Since by the choice of step size we may assume $\eta_j\le \frac{1}{3LH_{k(t)+1}}$, for any time instance in the time interval $[\tau_{k(t)}, \tau_{k(t)+1})$, we have shown that  
		\begin{align}\nonumber
			V_t\leq \frac{1}{3H_{k(t)+1}}\sum_{j=\tau_{k(t)}}^{t-1}V_j +H_{k(t)+1}\sum_{j=\tau_{k(t)}}^{t-1}3\eta_j^2\big(B^2h_j+\sigma^2+G^2\big).
		\end{align}
		Finally, if we recursively unroll $V_j, j=\tau_{k(t)},\ldots,t-1$ as in the proof of Lemma \ref{lem:consensus}, we obtain
		\begin{align*}
			V_t&\le (1+\frac{1}{3H_{k(t)+1}})^{H_{k(t)+1}}\times  H_{k(t)+1}\sum_{j=\tau_{k(t)}}^{t-1}3\eta_j^2\big(B^2h_j+\sigma^2+G^2\big)\\
			&\le H_{k(t)+1}\sum_{j=\tau_{k(t)}}^{t-1}6\eta_j^2\big(B^2h_j+\sigma^2+G^2\big).
		\end{align*}
	\end{proof}
	
	\begin{lemma}\label{eq:algebra}
		Let $\bbx=\frac{1}{n}\sum_{i = 1}^n\bx_i$. Then, for any $\bx'\in \R^d$, $\sum_{i = 1}^n\|\bx_i-\bx'\|^2 = \sum_{i = 1}^n\|\bx_i-\bbx\|^2+n\|\bbx-\bx'\|^2.$ In particular, $\sum_{i = 1}^n\|\bx_i-\bbx\|^2\le \sum_{i = 1}^n\|\bx_i\|^2$.
	\end{lemma}
	\begin{proof}
		We have,
		\begin{align*}
			&\sum_{i = 1}^n\|\bx_i-\bx'\|^2=\sum_{i = 1}^n\|\bx_i-\bbx+\bbx-\bx'\|^2\\
			&=\sum_{i = 1}^n\|\bx_i-\bbx\|^2\!+\!n\|\bbx-\bx'\|^2\!-\!\!\sum_{i = 1}^n\langle \bx_i-\bbx,\bbx-\bx'\rangle\cr 
			&=\sum_{i = 1}^n\|\bx_i-\bbx\|^2\!+\!n\|\bbx-\bx'\|^2.
		\end{align*}
		The second inequality holds by choosing $\bx'=0$.
	\end{proof}

	\begin{lemma}\label{prop:comvex,smooth}
		Let $f$ be a $L$-smooth convex function. Then, for any $\bx,\by \in \R^d$, we have
		\begin{align}\nonumber
			&f(\bx)-f(\by) + \frac{L}{2}\norm{\bx-\by}^2_2 \ge \lin{\nabla f(\bx),\bx-\by},\cr %\label{eq:lsmooth}
			&\|\nabla\! f(\bx)\!-\!\nabla\! f(\by)\|^2\!\le\! 2L(f(\by)\!-\!f(\bx)\!-\!\langle\nabla\! f(\bx),\by \!-\!\bx \rangle). %\label{eq:gnorm}
		\end{align}
	\end{lemma}
	\begin{proof}
		The first inequality is an immediate consequence of the $L$-smoothness property. To show the second inequality, let us define $\bz  = \by -\frac{1}{L}(\nabla f(\by)-\nabla f(\bx))$. Then,
		\begin{align*}
			&f(\bz)\ge f(\bx)+\lin{\nabla f(\bx),\bz-\bx},\\
			&f(\bz)\le f(\by)+\lin{\nabla f(\by),\bz-\by}+\frac{L}{2}\|\by -\bz\|^2.
		\end{align*}
		Therefore,
		\begin{align*}
			&f(\bx)+\Big\langle\nabla f(\bx),\by -\frac{1}{L}(\nabla f(\by)-\nabla f(\bx))-\bx\Big\rangle\cr
			&\qquad\le f(\by)-\Big\langle\nabla f(\by), \frac{1}{L}(\nabla f(\by)-\nabla f(\bx))\Big\rangle\cr 
			&\qquad+\frac{1}{2L}\| (\nabla f(\by)-\nabla f(\bx))\|^2.
		\end{align*}
		Rearranging the terms completes the proof.
	\end{proof}
	
		\subsection{Choice of $\beta$ in Example~\ref{exp:2}}
		Here we prove that in Example~\ref{exp:2}, we can choose $\beta =  a\lceil \frac{24L}{\mu}\rceil^s\cdot \frac{12L}{\mu}+1$ in order to satisfy condition 1) in Corollary~\ref{corr:strongly-convex}, i.e. $H_i\le \frac{\mu (\beta + \sum_{j = 1}^{i-1}H_j)}{12L},\ \forall i$. Since $a = \mathcal{O}(n^{-\frac{s+1}{2}}T^{\frac{1-s}{2}})$, the overall convergence rate is still $\mathcal{O}(\frac{1}{nT})$.
		
		\begin{proof}
			Let $k = \lceil \frac{24L}{\mu}\rceil$, then $\beta \ge H_k\cdot \frac{12L}{\mu}+1$. For all $i\le k$, we have 
			\begin{align*}
				H_i\le H_k<\frac{\mu\beta}{12L}<\frac{\mu (\beta + \sum_{j = 1}^{i-1}H_j)}{12L}.
			\end{align*}
			For all $k\le i\le T$, we would prove by induction that $a\cdot i^s\le \frac{\mu (\beta + \sum_{j = 1}^{i-1}H_j)}{12L}$, thus concluding the proof.
			
			In fact, for the base case $i=k$, we have 
			\begin{align*}
				a\cdot \lceil \frac{24L}{\mu}\rceil^s<\frac{\mu\beta}{12L}<\frac{\mu (\beta + \sum_{j = 1}^{k-1}H_j)}{12L}.
			\end{align*}
			For inductive step, assume for  some $k\le i\le T$, we have $a\cdot i^s\le \frac{\mu (\beta + \sum_{j = 1}^{i-1}H_j)}{12L}$, then 
			\begin{align*}
				&a\cdot (i+1)^s\le \frac{\mu (\beta + \sum_{j = 1}^{i}H_j)}{12L}\\
				\Leftarrow& a\cdot (i+1)^s-a\cdot i^s \le \frac{\mu H_i}{12L}\\
				\Leftarrow& a\cdot (i+1)^s-a\cdot i^s \le \frac{\mu a\cdot i^s}{24L}\\
				\Leftarrow& (\frac{i+1}{i})^s\le \frac{\mu}{24L}+1\\
				\Leftarrow& i\ge \frac{1}{(\frac{\mu}{24L}+1)^{\frac{1}{s}}-1}\\
				\Leftarrow& i\ge \frac{1}{1+\frac{\mu}{12Ls}-1}\\
				\Leftarrow& i\ge \lceil \frac{24L}{\mu}\rceil = k.
			\end{align*}
			By induction we conclude that for all $k\le i\le T$, we also have $H_i\le a\cdot i^s\le \frac{\mu (\beta + \sum_{j = 1}^{i-1}H_j)}{12L}$.
		\end{proof}

\end{document}